\documentclass{ecai}

\usepackage{graphicx}
\usepackage{latexsym}
\usepackage{times}
\usepackage{soul}
\usepackage{url}
\usepackage[hidelinks]{hyperref}
\usepackage[utf8]{inputenc}
\usepackage{rotating}
\usepackage[graphicx]{realboxes}
\usepackage{amsmath}
\usepackage{amsthm}
\usepackage{booktabs}
\usepackage{algorithm}
\usepackage{algorithmic}
\usepackage{amssymb}
\usepackage{xcolor}
\usepackage{colortbl}
\usepackage{tikz}
\usepackage{comment}
\usepackage[switch]{lineno}
\usepackage{multicol}

\newenvironment{exs}[1]{\begin{trivlist}\item[]\textbf{Example \ref{#1} (Cont)}}{\end{trivlist}}

\newtheorem{prop}{Proposition}
\newtheorem{proposition}[prop]{Proposition}
\newtheorem{defn}{Definition}
\newtheorem{definition}[defn]{Definition}

\newtheorem{exmp}{Example}
\newtheorem{example}[exmp]{Example}
\newtheorem{lem}{Lemma}
\newtheorem{lemma}[lem]{Lemma}
\newtheorem{mytheorem}{Theorem}

\newtheorem{property}{Property}

\newcommand{\bF}{{\mathbf{L}}}
\newcommand{\F}{{\mathtt{F}}}
\renewcommand{\d}{{\mathtt{dom}}}

\newcommand{\C}{\mbox{$\cal C$}}

\newcommand{\imp}{\stackrel{_{_{\mathcal{C}}}}{\rightarrow}}

\newcommand{\fml}[1]{{\mathcal{#1}}}

\newcommand{\mbf}[1]{\ensuremath{#1}}
\newcommand{\mbb}[1]{\ensuremath\mathbb{#1}}

\newcommand{\cov}{\mathtt{cov}}

\definecolor{maroon}{cmyk}{0,0.87,0.68,0.32}
\newcolumntype{a}{>{\columncolor{maroon!20}}c}
\newcolumntype{b}{>{\columncolor{maroon!10}}c}
\begin{document}

\begin{frontmatter}

\title{Abductive Explanations of Classifiers under Constraints: \\ Complexity and Properties}


\author[A]{\fnms{Martin}~\snm{Cooper}\thanks{Corresponding Author. Email: Martin.Cooper@irit.fr}}
\author[B]{\fnms{Leila}~\snm{Amgoud}}
\address[A]{University of Toulouse 3, France}
\address[B]{CNRS, IRIT France}
\begin{abstract}
Abductive explanations (AXp's) are widely used for understanding decisions of classifiers. 
Existing definitions are suitable when features are \emph{independent}. However, we show that 
ignoring constraints when they exist between features may lead to an explosion in the number 
of redundant or superfluous AXp's. 
We propose three new types of explanations that take into account constraints and that can be 
generated from the whole feature space or from a sample (such as a dataset). 
They are based on a key notion of coverage of an explanation, the set of instances it explains. 
We show that coverage is powerful enough to discard redundant and superfluous AXp's.  
For each type, we analyse the complexity of finding an explanation and investigate its 
formal properties. The final result is a catalogue of different forms of AXp's with different 
complexities and different formal guarantees. 
\end{abstract}

\end{frontmatter}
\section{Introduction}

Given a decision of a classifier, a user may want, and may even have a legal right to, an explanation
of this decision. Concrete examples include an explanation as to why a loan/job/visa application was 
refused or why a medical diagnosis was made. See \cite{Miller2019,molnar2020} for more on explainability and interpretability.

The majority of existing explanation functions explain a decision in terms of relevance of the input features. One of the most studied types of feature-based explanations  is the so-called \textit{abductive explanation} (AXp), or \textit{prime implicant explanation} \cite{darwiche20,LiuL23,darwiche18}. It provides a (minimal) sufficient reason for the decision. 

In the literature, AXp's have two sources: they are generated either from a subset of instances as done by the two prominent 
explanation functions Anchors \cite{Ribeiro0G18} and LIME \cite{Ribeiro16} and those introduced in \cite{Amgoud23a,Amgoud23b},  or from the whole feature 
space (eg., \cite{marquis22a,AudemardBBKLM22,AudemardLMS23,CP21,darwiche20,marquis22b,IgnatievNM19,joao19b,darwiche18}). 
Whatever the source, features are \textit{implicitly} assumed to be \textit{independent}.  
However, constraints on values that features may take are ubiquitous in almost 
all real-world applications including analysis of election results, justifying medical treatments, etc. 

Constraints have been extensively studied in databases where several types have been distinguished \cite{Thalheim}. 
In the context of classifiers and their abductive explanations, 
we focus on two categories: \textit{integrity constraints} (IC) and \textit{dependency constraints} (DC).  
The former are of two types: i) they may express impossible assignments of values to features like 
``men cannot be pregnant'', here ICs impact locally individual instances, ii) global constraints preventing the co-existence of two 
or more instances such as ``no two distinct students may have the same ID card value''.
When such constraints exist,  the feature space necessarily contains \textit{impossible instances}. 
Dependency constraints are a specific sub-type of the first type of IC. They express the following: if some attributes take specific values, then other attributes take also specific values. 
Examples of DCs are: "a person who is pregnant is necessarily a woman" and ``if it rains, then the road is certainly wet". 
This type of constraint may exist between features of  feasible instances. Therefore, 
they may lead to dependencies between AXp's, and thus redundancies (as some follow from others). 

\vspace{0.1cm}

In \cite{CP21,GorjiAAAI22}, ICs (in the sense of constraints on possible feature vectors) were considered when generating abductive explanations while DCs (in the sense of dependencies between AXp's) were totally ignored.  
In this paper, we show that disregarding dependency constraints when explaining decisions may lead to exponentially more AXp's many of which may be redundant or superfluous. To bridge this gap, we investigate explanation functions that generate AXp's while taking into account both IC and DC constraints. Our contributions are fourfold: 

The first consists of proposing three novel types of abductive explanation that deal with constraints. They are generated from the whole set of instances that satisfy the constraints,
thus discarding any instance that violates an integrity constraint. However, this is not sufficient for considering dependencies expressed by DCs. As a solution, the new types of explanation are based on the key notion of \emph{coverage} of an explanation, i.e., the set of all 
instances it explains. Coverage is powerful enough to capture those constraints, and ensures the independence of explanations of every decision.

The second contribution consists of a thorough analysis of the complexity of explaining a decision.   
We show that finding a prime-implicant explanation becomes computationally much more challenging in a constrained setting. 

The third contribution consists of proposing a paradigm for making the three solutions feasible. The idea is to avoid 
exhaustive search by examining a sample of the constrained feature space. We adapt the three types of explanations  
and show that the worst-case complexity of finding sample-based explanations is greatly reduced.

The fourth contribution consists of introducing desirable properties that an explanation function should 
satisfy, then comparing and analysing the novel functions against them. The results show, in particular, that 
when explanations are generated from a sample, complexity is greatly reduced but at the cost of violating a 
desirable property, which ensures a kind of global coherence of the set of all explanations that may be returned by a function.

\vspace{0.1cm}

The paper is structured as follows: Section~\ref{motivation} recalls previous definitions of AXp and    
Section~\ref{constraints} discusses their limits. 
Section~\ref{sec:CPIXp} defines the three novel types of AXp's,  
Section~\ref{complexity} analyses their complexity,  and 
Section~\ref{sec:d-Xps} investigates their sample-based versions. 
Section~\ref{sec:axioms} introduces properties of explainers and 
analyses the discussed functions against them. The last section concludes.   
All proofs are included in the appendix.
\section{Background} \label{motivation}

Throughout the paper, we consider a \textit{classification theory} as a tuple made of a finite 
set $\F$ of \textit{features} (also called attributes), a function $\d$ which returns 
the \textit{domain} of every feature $f \in \F$,
where $\d(f)$ is finite and $|\d(f)| > 1$, and a finite 
set $\mathtt{Cl}$ of \textit{classes} with $|\mathtt{Cl}| \geq 2$.  
We call \textit{literal} any pair $(f,v)$ where $f \in \F$ and $v \in \d(f)$. 
A \textit{partial assignment} is any set of literals with each feature in $\F$ occurring at most once; 
it is called an \textit{instance} when every feature appears once.  
We denote by $\mbb{E}$ the set of all possible partial assignments and by 
$\mbb{F}$ the \textit{feature space}, i.e., the set of all instances. 
For all $E, E' \in \mbb{E}$, the notation $E(E')$ is a shorthand for $E \subseteq E'$. 
The reason for this notation is that if $E$ is a partial assignment, then $E$ can be viewed as a predicate on instances $x$: $E(x)$ means that $x$ agrees with $E$ on the subset of features on which it is defined.

\begin{definition}[Theory]
	A \emph{classification theory} is a tuple $\langle \F, \d, \mathtt{Cl} \rangle$.
\end{definition}

We consider a classifier $\kappa$, which is a 
function mapping every instance in $\mbb{F}$ to a class in the set $\mathtt{Cl}$. We make the reasonable assumption that $\kappa$ can be evaluated in polynomial time.

\vspace{0.15cm}

Abductive explanations (AXp) answer questions of the form: why is instance $x$ assigned outcome $c$ by 
classifier $\kappa$? They are  partial assignments, which are sufficient for ensuring the prediction $c$.  
We recall below the definition of AXp ~\cite{CP21,darwiche18}.

\begin{definition}[wAXp, AXp]\label{def:axp}
	Let $x \in \mbb{F}$. 
	A \emph{weak AXp} (wAXp) of $\kappa(x)$ is a partial assignment $E \in \mbb{E}$ such that:
	\begin{itemize} 
		\item $E(x)$, 
		\item $\forall y \in \mbb{F}. (E(y) \rightarrow (\kappa(y)=\kappa(x)))$. 
	\end{itemize} 
	An \emph{AXp} of $\kappa(x)$  is a subset-minimal weak AXp.
\end{definition}

\begin{example}\label{ex0}
	Suppose that $\F = \{f_1, f_2\}$, with $\d(f_1) = \d(f_2) = \{0,1\}$, and $\mathtt{Cl} = \{0,1\}$.  	
	Consider the classifier $\kappa_1$ such that for any $x \in \mbb{F}$, $\kappa_1(x) = (f_1,1) \vee (f_2,1)$. 
	Its predictions are summarized in the table below. 

\begin{nolinenumbers}
\begin{multicols}{2}
	\begin{center}
		\begin{tabular}{|c|cc|c|}
			\hline
			&$f_1$ & $f_2$ & $\kappa_1(x_i)$ \\
			\hline 
\rowcolor{maroon!20}$x_1$ & 0 & 0 &  0 \\
			$x_2$ & 0 & 1 &  1 \\
			$x_3$ & 1 & 0 &  1 \\
\rowcolor{maroon!20}$x_4$ & 1 & 1 &  1 \\
			\hline
		\end{tabular}
	\end{center}

\begin{itemize}
	\item $E_1 = \{(f_1,1)\}$ 
	\item $E_2 = \{(f_2,1)\}$
	\item $E_3 = \{(f_1,1),(f_2,1)\}$ 
	\item $E_4 = \{(f_1,0),(f_2,0)\}$
\end{itemize}	
\end{multicols}
\end{nolinenumbers}

\noindent The decision $\kappa_1(x_4)$ has three weak abductive explanations $E_1, E_2, E_3$ 
and two AXp's: $E_1, E_2$. The decision $\kappa_1(x_1)$ has a single wAXp/Axp, namely $E_4$.
\end{example}

Explaining decisions made by classifiers is 
in general \textit{not tractable} (assuming P$\neq$NP) as shown in 
\cite{CP21,COOPER2023}. 

\begin{property}(\cite{CP21,COOPER2023}) \label{axp-complexity}
The problem of testing whether a set $E \in \mbb{E}$ is a weak AXp is co-NP-complete. 
The problem of finding one AXp is in $\text{FP}^{\text{NP}}$, the class of 
functional problems that can be solved by a polynomial number of calls to a SAT oracle. 
\end{property}

In \cite{CP21,COOPER2023,GorjiAAAI22}, the authors investigated 
\textit{AXp's under constraints}. 
They assume, as we do in this paper, a finite set $\C$ of constraints
between features, which can be considered as a predicate.  
For any partial assignment $E$, $\C(E)$ means that $E$ satisfies all the constraints 
in $\C$, $\neg \C(E)$ means $E$ violates at least one constraint, 
and if $\C = \emptyset$ then $\C(y) \equiv \top$. 
They took into account constraints in the definition of an AXp by checking only feasible instances which gave the following definition
\cite{CP21,COOPER2023,GorjiAAAI22}. 

\begin{definition}[AXpc]\label{axpc}
Let $x \in \mbb{F}$ be such that $\C(x)$, where $\C$ is a finite set of constraints.
A \emph{weak AXpc} (wAXpc) of $\kappa(x)$ is a partial assignment $E \in \mbb{E}$ such that:
	
\begin{itemize} 
	\item $E(x)$, 
	\item $\C(E)$,
	\item $\forall y \in \mbb{F}. (\C(y) \wedge E(y) \rightarrow (\kappa(y)=\kappa(x)))$. 
\end{itemize}
An \emph{AXpc} of $\kappa(x)$ is a subset-minimal weak AXpc.
\end{definition}

\begin{exs}{ex0}
Assume the existence of the constraint $f_1 \wedge \neg f_2 \rightarrow \bot$, 
which means the instance $x_3$ is impossible. 
According to the above definitions, the decision $\kappa_1(x_1)$ has two weak AXpc's: 
$E_4$ and $E_5 = \{(f_2,0)\}$, and one AXpc which is  $E_5$. 
\end{exs} 

The example shows an AXpc that is a subset of an AXp. The next result 
confirms this link between the two notions.

\begin{proposition}\label{axp-vs-axpc}
Let $x \in \mbb{F}$ such that $\C(x)$, and $E \in \mbb{E}$.  
If $E$ is an AXp of $\kappa(x)$, then $\exists E' \in \mbb{E}$
s.t. $E' \subseteq E$ and $E'$ is an AXpc of $\kappa(x)$.   
\end{proposition}


However, an AXpc of a decision is not always related to an AXp of the same decision as shown 
below.  

\begin{example}\label{ex0bis}
	Consider the theory of Example~\ref{ex0} and   	
	the classifier $\kappa_2$ such that for $x \in \mbb{F}$, $\kappa_2(x) = \neg f_1$. 
	Suppose $\C$ contains one constraint, $f_1 \wedge \neg f_2 \rightarrow \bot$, which 
	is violated by $x_3$. 
\begin{nolinenumbers}
\begin{multicols}{2}
	\begin{center}
		\begin{tabular}{|c|cc|c|}
			\hline
			&$f_1$ & $f_2$ & $\kappa_2(x_i)$ \\
			\hline 
\rowcolor{maroon!20}$x_1$ & 0 & 0 &  1 \\
			$x_2$ & 0 & 1 &  1 \\
			$x_3$ & 1 & 0 &  0 \\
            $x_4$ & 1 & 1 &  0 \\
			\hline
		\end{tabular}
	\end{center}

\begin{itemize}
	\item $E_1 = \{(f_1,0)\}$ 
	\item $E_2 = \{(f_2,0)\}$ 
	\item $E_3 = \{(f_1,0),(f_2,0)\}$
\end{itemize}	
\end{multicols}
\end{nolinenumbers}

\noindent The decision $\kappa_2(x_1)$ has $E_1$ as its sole AXp. However, 
it has two AXpc's: $E_1$ and $E_2$. 
\end{example}
\section{Limits}\label{constraints} 

The definition of AXp implicitly assumes \textit{independence of features}, i.e.  there are no constraints between the values they may take. 
The definition of an AXpc accounts for constraints but only partially. 
In what follows, we discuss three undesirable consequences of ignoring  
dependency constraints: existence of \textit{superfluous} explanations, \textit{redundancy} of explanations and \textit{explosion in their number}. 

\paragraph{Superfluous Explanations.} We show next that ignoring constraints may lead to generating \textbf{gratuitous} explanations. 


\begin{exs}{ex0bis}
Recall that the decision $\kappa_2(x_1)$ has two AXpc's: 
$E_1 = \{(f_1,0)\}$ and $E_2 = \{(f_2,0)\}$. From the definition 
of $\kappa_2$ ($\forall x \in \mbb{F}$, $\kappa_2(x) = \neg f_1$), 
it follows that $E_1$ is correct while $E_2$, although
logically correct, is superfluous. 
The correlation between $E_2$ and $\kappa_2(x_1)$ is due to 
the dependency constraint stating: whenever $f_2$ takes the value 0, 
$f_1$ takes the same value (and consequently, $\kappa_2$ assigns 1 
to the corresponding instance).
\end{exs}

\paragraph{Redundancy.} 
The following example shows that some AXp's may be \textit{redundant} with respect to others due to dependency constraints between values of features. 

\begin{example}\label{ex1}
Assume that $\F = \{f_1, f_2\}$, where $f_1$ and $f_2$ stand respectively for gender (0 for male, 1 for female) 
and being pregnant, $\d(f_1) = \d(f_2) = \{0,1\}$, and $\mathtt{Cl} = \{0,1\}$.  
Consider the constraint stating that only women can be pregnant. 
The instance $\{(f_1, 0),(f_2, 1)\}$ is then impossible.  	
Consider the classifier $\kappa_3(x)=f_1 \lor f_2$ whose predictions for the possible instances are given in the table below.  
\begin{nolinenumbers}
\begin{multicols}{2} 
\begin{center}
		\begin{tabular}{|c|cc|c|}
			\hline
			&$f_1$ & $f_2$ & $\kappa_3(x_i)$ \\
			\hline 
			$x_1$ & 0 & 0 &  0 \\	
			$x_2$ & 1 & 0 &  1 \\
			\rowcolor{maroon!20}$x_3$ & 1 & 1 &  1 \\
			\hline
		\end{tabular}
\end{center}
\begin{itemize}
	\item $E_1 = \{(f_1,1)\}$ 
	\item $E_2 = \{(f_2,1)\}$
	\item $E_3 = \{(f_1,1),(f_2,1)\}$ 
\end{itemize}	
\end{multicols}
\end{nolinenumbers}
\noindent The decision $\kappa_3(x_3)$ has two AXpc's: $E_1$ and $E_2$.  
Note that the two explanations are not independent, and $E_2$ is 
somehow redundant with $E_1$ since decisions concerning women in general hold for those who are pregnant. 	
\end{example}

\paragraph{Exponential number of explanations.} 
In the above examples, only one explanation is redundant. However, the number may be exponential as shown in the next
example. 
	
\begin{example}\label{ex2}
Let $\F = \{f_1, \ldots, f_n\}$, $\forall i \in \{1,\ldots, n\}$, $\d(f_i) = \{0,1\}$,
$\mathtt{Cl} = \{0,1\}$, and let $\kappa_4$ be the classifier  
$\kappa_4(x) = f_n$.
Assume also that there is a constraint:
$f_n \equiv (\sum_{i=1}^{n-1} f_i \geq \lfloor n/2 \rfloor)$.
Let $x = \{(f_i,1) \mid i=1, \ldots, n\}$, so $\kappa_4(x)=1$. 
The decision $\kappa_4(x)=1$ has $n \choose k$ AXpc's, where $k=\lfloor \frac{n}{2} \rfloor$: all size-$k$ subsets of $\{(f_i, 1) \mid i=1, \ldots, n-1\}$, as well as the AXpc $\{(f_n,1)\}$.
Observe that $\{(f_n,1)\}$ subsumes all other AXpc's. 
Therefore, one could discard all explanations other than $\{(f_n,1)\}$ since they are superfluous.
\end{example}

\vspace{0.1cm}

\textbf{To sum up,} defining abductive explanations that deal with dependency 
constraints remains a challenge that has never been addressed in the literature. 
We propose in the next sections the first solutions and investigate their complexity and formal properties. 
\section{Explanations and feature-space coverage} \label{sec:CPIXp}

We revisit in this section the definition of abductive explanations 
for constrained settings. In the rest of the paper, we assume a fixed 
but arbitrary classification theory $\langle\F, \d,\mathtt{Cl}\rangle$ and 
a \textbf{finite} set $\C$ of constraints on the theory, and more 
precisely on its set $\mbb{E}$ of partial assignments. 
For $E \in \mbb{E}$, the notation $\C(E)$ means $E$ satisfies all constraints in $\C$, 
$\neg \C(E)$ means $E$ violates at least one constraint, and 
$\mbb{F}[\C] = \{x \in \mbb{F} \mid \C(x)\}$, i.e., the set of instances in $\mbb{F}$ that satisfy the constraints. The set $\C$ satisfies the following properties: 

\begin{description}
    \item [(C1)] $\mbb{F}[\C] \neq \emptyset$ (constraints in $\C$ 
          can be satisfied all together). 
    \item [(C2)] Let $E, E' \in \mbb{E}$. If $E \subseteq E'$, then $\C(E') \to \C(E)$.   
\end{description}
We consider a classifier $\kappa$ which is a function mapping every instance in $\mbb{F}[\C]$ to a class in $\mathtt{Cl}$. 
We assume that the test $x \in \mbb{F}[\C]$ and the calculation of $\kappa(x)$ are  
polynomial. 

\vspace{0.1cm}

We have seen that there are two types of constraints. Integrity constraints describe  
impossible assignments of values. The definition of an AXpc takes them into account 
by checking instance feasibility. Our approach starts by removing all unrealistic 
instances and focuses only on $\mbb{F}[\C]$. 
However, we have seen in the previous section that this solution is not sufficient for 
dealing with dependencies between partial assignments that follow from constraints $\C$. 
Before showing how we deal with such dependency constraints, let us first define them.

\begin{definition}[DC]
A \textit{dependency constraint} (DC) is any formula of the form 
$E \to E'$ 
such that:
	\begin{itemize} 
		\item $E, E' \in \mbb{E}\setminus\{\emptyset\}$, 
		\item $E \neq E'$, 
		\item For any $x \in \mbb{F}[\C]$, if $E(x)$ then $E'(x)$.
	\end{itemize} 	
    We denote by $\C^*$ the set of all such constraints. 
\end{definition}

DC's are defined on the entire set $\mbb{F}[\C]$ of feasible instances. 
A DC $E \to E'$ means that whenever $E$ holds, $E'$ holds as well. 
In Example~\ref{ex1}, the constraint $\{(f_2,1)\} \to \{(f_1,1)\}$ means 
that when the feature $f_2$ takes the value 1, the feature $f_1$ necessarily takes the same value.

\vspace{0.1cm}
Our approach takes advantage of such information for reducing the number of abductive explanations 
by avoiding dependent explanations, and therefore discarding redundant or superfluous ones. 
Before defining the novel notions of explanation, let us first introduce some useful notions. The first 
one is the \textit{coverage} of a partial assignment, which is the set of instances it covers.

\begin{definition}[Coverage]
Let $X \subseteq \mbb{F}$ and $E \in \mbb{E}$.
The  \emph{coverage} of $E$ in $X$ is the set 
$\cov_X(E) =  \{x \in X \mid E(x)\}$. 
When $X = \mbb{F}[\C]$, we write $\cov(E)$ for short.
\end{definition}

\begin{exs}{ex1}
For $E = \{(f_1,1)\}$, $\cov(E) = \{x_2, x_3\}$.
\end{exs}

The second notion, which is crucial for the new definition of explanation, is a \textit{subsumption} relation defined as follows. 

\begin{definition} \label{def:subsumption}
Let $X \subseteq \mbb{F}$ and $E,E' \in \mbb{E}$. 
We say that $E'$ \emph{subsumes} $E$ in $X$ 
if $\forall x \in X. (E(x) \rightarrow E'(x))$. 
$E'$ \emph{strictly subsumes} $E$ in $X$ if $E'$ subsumes $E$ in $X$ but
$E$ does not subsume $E'$ in $X$. 
\end{definition}

We show that subsumption is closely related to coverage. 

\begin{proposition}\label{cov}
Let $X \subseteq \mbb{F}$ and $E, E' \in \mbb{E}$. 
\begin{itemize}
\item The following statements are equivalent. 
\begin{itemize}
	\item $E'$ subsumes $E$ in $X$.
	\item $\cov_X(E) \subseteq \cov_X(E')$.
\end{itemize}
\item If $E' \subset E$, then $E'$ subsumes $E$ in $X$. The converse does not always hold. 
\item If $E \neq E'$, then $\cov_\mbb{F}(E) \neq \cov_\mbb{F}(E')$.
\end{itemize}
\end{proposition}

\begin{exs}{ex1}
The partial assignment $E_1 = \{(f_1,1)\}$ strictly subsumes $E_2 = \{(f_2,1)\}$ in 
the space $\mbb{F}[\C]$. Indeed, $\cov(E_1) = \{x_2, x_3\}$ and 
$\cov(E_2) = \{x_3\}$.
\end{exs}

The subsumption relation is \textbf{not monotonic} meaning that a partial assignment $E$ may subsume another (say $E'$) in a set of instances $X$ but not in some $Y \supset X$ as shown in the following example. 

\begin{exs}{ex0} 
Assume again the existence of the constraint $f_1 \wedge \neg f_2 \rightarrow \bot$, 
which means $\mbb{F}[\C] = \{x_1, x_2, x_4\}$. 
Let  $E_1 = \{(f_1,0)\}$ and $E_2 = \{(f_2,0)\}$. 
Note that $E_2$ subsumes $E_1$ in $X = \{x_1\}$ but not in $\mbb{F}[\C]$.  
\end{exs}

We show that every constraint in $\C^*$ can be expressed as a subsumption relation, 
which holds in any subset of instances.

\begin{proposition}\label{cov-cont}
Let $E, E' \in \mbb{E}$. If $E \to E' \in \C^*$, then $\forall X \subseteq \mbb{F}[\C]$, 
$E'$ subsumes $E$ in $X$. 
\end{proposition}

	

Let us now introduce our novel notion of \textit{coverage-based prime-implicant explanation} (CPI-Xp). 
The first idea is to generate AXp's from the set of instances that satisfy the available constraints,  
thus discarding impossible instances. Furthermore, it selects AXp's which subsume the others, thus taking 
into account DCs. This is equivalent to selecting AXp's that apply to more instances in $\mbb{F}[\C]$.

\begin{definition}[CPI-Xp] \label{def:CPI}
Let $x \in \mbb{F}[\C]$. 	
A \emph{coverage-based PI-explanation} (CPI-Xp) of $\kappa(x)$ is 
any $E \in \mbb{E}$ such that:  
\begin{itemize}
    \item $E(x)$, 
    \item $\forall y \in \mbb{F}[\C]. (E(y) \rightarrow (\kappa(y)=\kappa(x)))$, 
    \item $\nexists E' \in \mbb{E}$ such that $E'$ satisfies the above conditions and 
    strictly subsumes $E$ in $\mbb{F}[\C]$.  
\end{itemize}
\end{definition}

While a CPI-Xp is clearly a weak AXpc, the two notions do not always coincide when the set of constraints is empty. 

\begin{exs}{ex0}
Let $\C = \emptyset$ ($\mbb{F}[\C] = \mbb{F}$). 
Note that 
$\cov(E_1)$ = $\{x_3,x_4\}$, 
$\cov(E_2)$ = $\{x_2,x_4\}$,
$\cov(E_3)$ = $\{x_4\}$ showing that $E_3$ is strictly subsumed by $E_1$ and $E_2$ in $\mbb{F}$. 
So $E_3$ is not a CPI-Xp of $\kappa_1(x_4)$ while it is a wAXpc and a wAXp.
\end{exs}

Let us now show how the notion of CPI-Xp solves the three problems discussed in the previous section. 

\begin{exs}{ex0bis}
Recall that $f_1 \wedge \neg f_2 \rightarrow \bot$ is a constraint,  
so $\{(f_1,1)\} \to \{(f_2,1)\} \in \C^*$ and 
$\mbb{F}[\C] = \{x_1, x_2, x_4\}$. 
The decision $\kappa_2(x_1)$ has three weak AXpc's in $\mbb{F}[\C]$: 
$E_1 = \{(f_1,0)\}$, 
$E_2 = \{(f_2,0)\}$, and 
$E_3 = \{(f_1,0), (f_2,0)\}$. 
Note that $\cov(E_2) = \cov(E_3) = \{x_1\} \subset \cov(E_1) = \{x_1, x_2\}$. 
So $E_1$ is the sole CPI-Xp of $\kappa_2(x_1)$, discarding the superfluous AXpc $E_2$. 
This shows that subsumption is powerful enough to detect \textbf{gratuitous correlations} between features and decisions.
\end{exs}

\begin{exs}{ex1}
	Recall that $\C^* = \{E_2 \to E_1\}$, $E_1 = \{(f_1,1)\}$, $E_2 = \{(f_2,1)\}$, and  
	$\mbb{F}[\C] = \{x_1, x_2, x_3\}$. 
	The decision $\kappa_3(x_3)$ has two AXpc's: $E_1$ and $E_2$. 
	However, it has a single coverage-based PI-explanation, namely $E_1$ which subsumes $E_2$. 
	The redundant AXpc $E_2$ is thus discarded. 
\end{exs}

\begin{exs}{ex2}
Recall that $x = \{(f_i,1) \mid 1 \leq i \leq n\}$ and 
the decision $\kappa_4(x){=}1$ has a combinatorial number of AXpc's: 
all subsets of $\{(f_i, 1) \mid 1 \leq i \leq n{-}1\}$ of size $n \choose {\lfloor n/2 \rfloor}$. 
Due to the constraint  
$f_n \equiv (\sum_{i=1}^{n-1} f_i \geq \lfloor n/2 \rfloor)$, 
the decision $\kappa_4(x){=}1$ has a single CPI-Xp, namely $\{(f_n, 1)\}$. So, there is a \textbf{drastic 
reduction} in the number of explanations. 
\end{exs}	
 		

Despite a significant reduction in the number of abductive explanations (AXp's), a decision may still 
have several coverage-based PI-explanations. In what follows, we discuss two \textbf{criteria} to further reduce  
the number of CPI-Xps: \textit{conciseness} and \textit{generality}. 
Let us start with the first criterion. We  show that CPI-Xp may contain irrelevant  information.

\begin{example}\label{ex4}
Consider a theory made of three binary features $f_1, f_2, f_3$. 
Let $E_1 = \{(f_1,1),  (f_2,1)\}$ and $E_2 = \{(f_3,1)\}$. 
Assume $\C^* = \{E_1 \to E_2, E_2 \to E_1\}$, then $\cov(E_1) = \cov(E_2) =  \cov(E_1 \cup E_2)$. 
Suppose that $E_2$ is a CPI-Xp.  
Then $E_1 \cup E_2$ is also a CPI-Xp but is not subset-minimal.  
\end{example}


Concision of explanations is important given the well known cognitive limitations of human users when 
processing information~\cite{Miller56}. A common way for ensuring concision is to require \textit{minimality} 
in order to avoid  irrelevant information in an explanation. 

\begin{definition}[Minimal CPI-Xp] \label{def:mCPI}
Let $x \in \mbb{F}[\C]$. A \emph{minimal coverage-based PI-explanation} (mCPI-Xp) of $\kappa(x)$ is a subset-minimal CPI-Xp of $\kappa(x)$. 	
\end{definition}

\begin{exs}{ex4}
	The set $E_1 \cup E_2$ is not a minimal 
	CPI-Xp since $E_1$ and $E_2$ are CPI-Xps.
\end{exs}

Let us now turn our attention to the generality criterion, which concerns the coverage of an explanation. Example~\ref{ex4} shows that coverage-based PI-explanations of a decision may have exactly the same coverage. Indeed, if $E_1$ is a minimal CPI-Xp, then so is $E_2$, and both have the same coverage. We say that such explanations are \emph{equivalent}.  


\begin{definition}[Equivalence]\label{equivalence}
Let $X \subseteq \mbb{F}[\C]$. 
Two sets $E,E' \in \mbb{E}$ are \emph{equivalent} in $X$, 
denoted by $E \approx E'$, iff they subsume each other in the set $X$.
\end{definition}

\noindent \textbf{Notation:} Let $X$ be a set and $\approx$ an equivalence relation on $X$. A set of \textit{representatives} of $X$ is a subset of $X$ containing exactly one element of every equivalence class of $X$, i.e., one element 
among equivalent ones.

\vspace{0.1cm}

We propose next \emph{preferred coverage-based explanations} that consider only one mCPI-Xp among equivalent ones (obviously in $\mbb{F}[\C]$ since mCPI-Xp are produced from  $\mbb{F}[\C]$).

\begin{definition}[Preferred CPI-Xp] \label{def:pCPI}
Let $x \in \mbb{F}[\C]$.     
A \emph{preferred coverage-based PI-explanation} (pCPI-Xp) of $\kappa(x)$ is a representative of the set of miminal CPI-Xp's of $\kappa(x)$.
\end{definition}

\begin{exs}{ex4} 
Definition~\ref{def:pCPI} selects either $E_1$ or $E_2$ (but not both) as a pCPI-Xp. 
\end{exs}

The three novel notions of explanation are clearly related to each other. 
Furthermore, when the set of constraints is empty, AXpc explanations 
presented in \cite{CP21}  (see Def.~\ref{axpc}) coincide with both AXp's and 
minimal CPI-Xp's. In the general case, a mCPI-Xp is an AXpc but the converse does not hold. 
This confirms that AXpc deals only with integrity constraints and ignores 
dependency constraints. 
Before presenting a summary of the links, let us 
first introduce the notion of \textit{explanation function} or \textit{explainer}.

\begin{definition}[Explainer]
An \textit{explainer} is a function $\bF_y$ mapping every instance $x \in \mbb{F}[\C]$ 
into the subset of $\mbb{E}$ 
consisting of
$y$-explanations of the decision $\kappa(x)$, for 
$y \in \{\mbox{wAXp}, \mbox{AXp}, \mbox{wAXpc}, \mbox{AXpc}, \mbox{CPI-Xp}, \mbox{mCPI-Xp}, \mbox{pCPI-Xp}\}$. 
\end{definition}

\begin{proposition}\label{links}
Let $x \in \mbb{F}[\C]$.
\begin{enumerate}
	\item $\bF_{AXp}(x) \subseteq \bF_{wAXp}(x)$,
	\item $\bF_{CPI-Xp}(x) \subseteq \bF_{wAXpc}(x)$,
	\item $\bF_{mCPI-Xp}(x) \subseteq \bF_{AXpc}(x) \subseteq \bF_{wAXpc}(x)$,
	\item $\bF_{pCPI-Xp}(x) \subseteq \bF_{mCPI-Xp}(x) \subseteq \bF_{CPI-Xp}(x)$, 
	\item If $\C = \emptyset$, then $\bF_{AXp}(x) = \bF_{AXpc}(x)  = \bF_{mCPI-Xp}(x)$.
\end{enumerate}
\end{proposition}

\textbf{To sum up}, we introduced three novel types of abductive explanations that better take into account constraints, and solve the three problems (superfluous, redundant,  
exponential number of explanations) of the existing definitions. 
\section{Complexity analysis}\label{complexity}

Let us investigate the computational complexity of the new types of explanation. 
We focus on the complexity of \textit{testing} whether a given partial assignment is a 
(minimal, preferred) coverage-based PI-explanation,  and the complexity of \textit{finding} one such explanation.

\vspace{0.1cm}

We first consider the computational problem of deciding if a partial assignment is a coverage-based PI-explanation (CPI-Xp). We show that the problem can be rewritten 
as an instance of $\forall\exists$SAT, the problem of testing the satisfiability of a quantified boolean formula of the form $\forall x \exists y \phi(x,y)$, where $x,y$ 
are vectors of boolean variables and $\phi$ is an arbitrary boolean formula with no 
free variables other than $x$ and $y$.
It is well known that $\forall\exists$SAT is complete for the complexity class $\Pi_2^{\rm P}$. It turns out that testing whether a weak AXpc is a coverage-based PI-explanation is also $\Pi_2^{\rm P}$-complete.




\begin{mytheorem} \label{prop:pi2}
The problem of testing whether a weak AXpc $E$ is a coverage-based PI-explanation is $\Pi_2^{\rm P}$-complete.
\end{mytheorem}

We now consider the problem of actually finding a coverage-based PI-explanation.
In the appendix we give an algorithm which returns one CPI-Xp of a given decision $\kappa(v)=c$.
It is based on the following idea: if a weak AXpc $E$ is not a coverage-based PI-explanation, then this is because there is a weak AXpc $E'$ that strictly subsumes $E$. We call such an $E'$ a counter-example to the hypothesis that $E$
is a coverage-based PI-explanation. Therefore, starting from a weak AXpc $E$, we can look for a counter-example $E'$: if no counter-example exists then we return $E$, otherwise we can replace $E$ by $E'$ and 
re-iterate the process. This loop must necessarily halt since there cannot be an infinite sequence
of partial assignments $E_1,E_2,\ldots$ such that $E_{i+1}$ strictly subsumes $E_i$ ($i=1,2,\ldots$).
We can be more specific: the following proposition shows that the number of iterations is bounded by the number of features. 

\begin{mytheorem} \label{prop:finding}
Let $n = |\F|$.
A CPI-Xp can be found by $n$ calls to an oracle for testing whether a given weak AXpc 
is a coverage-based PI-explanation.
\end{mytheorem}

It follows that the complexity of finding one coverage-based PI-explanation is essentially the same (modulo a linear factor) as testing whether a given weak AXpc is a coverage-based PI-explanation.

\vspace{0.15cm}

The following proposition shows that imposing minimality (for set inclusion) does not change the complexity. A minimal CPI-Xp can be found by $n$ calls to an oracle (for testing whether a given weak AXpc is a CPI-Xp) together with $2n$ calls to a SAT oracle. Note that since finding a preferred CPI-Xp (i.e. pCPI-Xp) consists of finding one minimal CPI-Xp, there is no change in complexity.  The difference between pCPI-Xp's and minimal CPI-Xp's becomes apparent when enumerating all  explanations: there can be many less pCPI-Xp's which is an advantage for the user. 

\begin{mytheorem} \label{prop:minimalCPI}
	Let $n = |\F|$. 
	A mCPI-Xp (resp. pCPI-Xp) can be found by $n$ calls to an oracle for testing whether a 
	given weak AXpc is a coverage-based PI-explanation together with $2n$ calls to a SAT oracle.
\end{mytheorem}
\textbf{To sum up}, we have shown that taking into account constraints (in the definition
of coverage-based prime implicants) may produce less-redundant explanations, but at 
the cost of a potential increase in computational complexity.
\section{Sample-based explanations}  \label{sec:d-Xps}

When a classifier is a black-box or a deep neural network which cannot be realistically written 
down as a function, the only algorithm for testing whether a set of literals is a (weak) AXp is 
an exhaustive search over the whole feature space. This explains why they are costly from a 
computational point of view. We have seen in the previous section that the computational complexity 
of the three novel types of explanations that deal with constraints (CPI-Xp, mCPI-Xp and pCPI-Xp) 
is even worse. In this section, we propose a paradigm for making the solutions feasible. The idea 
is to avoid the exhaustive search by examining only a sample of the feature space. The 
obtained explanations are approximations that can be obtained with lower complexity as we will see next. 

\vspace{0.15cm}

In this section, we concentrate on a sample (or dataset) $\fml{T} \subseteq \mbb{F}[\C]$ and the 
associated values of a black-box classifier $\kappa$. 
Note that $\fml{T}$ may be the dataset a classifier has been trained on, a dataset on which 
the classifier has better performance, or may be generated in a specific way as in \cite{Ribeiro16,Ribeiro0G18}. 
However, we assume that every possible class  in $\mathtt{Cl}$ is considered in the sample, i.e., 
it is assigned to at least one instance in $\fml{T}$:  
$\forall c \in \mathtt{Cl}, \exists x \in \fml{T} \mbox{ such that } \kappa(x) = c.$ 


\vspace{0.15cm}

In what follows, we adapt the definitions of the different explanations discussed in the previous sections, and add a suffix `d-' to indicate the new versions.  

\subsection{Abductive explanations based on samples} 
Recall that an AXp $E$ of a decision $\kappa(x)=c$ is a minimal subset of $x$ which guarantees the class $c$.
If $\kappa$ is a black-box function, then testing this definition for a given $E$ requires testing the exponential
number of assignments to the features not assigned by the partial assignment $E$. 
The following definition only requires us to test those instances in the sample $\fml{T}$.

\begin{definition}[d-wAXp, d-AXp]
Let $x \in  \mbb{F}$. 
A \emph{weak dataset-based AXp} (d-wAXp) of $\kappa(x)$ wrt to $\fml{T}$ 
is a partial assignment $E \in \mbb{E}$ such that $E(x)$ and 
$\forall y \in \fml{T}$, if $E(y)$ then $\kappa(y)=\kappa(x)$. 
A \emph{dataset-based AXp} (d-AXp) of $\kappa(x)$ is a subset-minimal d-wAXP of $\kappa(x)$.
\end{definition}

In other words, a d-AXp is mathematically equivalent to an AXp under the artificial constraint that the only allowed feature vectors are those in the dataset.

\begin{example}\label{ex6}
Assume a theory made of four binary features, a binary classifier $\kappa$ defined as follows: 
$\kappa(x) = (f_1 \land f_2) \lor (f_3 \land f_4)$.
Let $\fml{T}$ be a sample whose instances and their predictions are summarized in the table below.

\begin{nolinenumbers}
\begin{multicols}{2}
\begin{tabular}{|c|cccc|c|}
		\hline \hspace{-1mm}
		$\fml{T}$ \hspace{-1mm} & \hspace{-1mm} $f_1$ \hspace{-2mm} & \hspace{-2mm} $f_2$ \hspace{-2mm} & \hspace{-2mm} $f_3$ \hspace{-2mm} & \hspace{-2mm} $f_4$ \hspace{-2mm} & \hspace{-1mm} $\kappa$ \hspace{-1mm} \\
		\hline 
		\hspace{-1mm} $x_1$ \hspace{-1mm} & 0 & 1 & 0 & 0 & 0 \\
		\hspace{-1mm} $x_2$ \hspace{-1mm} & 0 & 1 & 0 & 1 & 0 \\
		\hspace{-1mm} $x_3$ \hspace{-1mm} & 0 & 1 & 1 & 0 & 0 \\
		\hspace{-1mm} $x_4$ \hspace{-1mm} & 0 & 0 & 1 & 1 & 1 \\
\rowcolor{maroon!20}		\hspace{-1mm} $x_5$ \hspace{-1mm} & 1 & 1 & 1 & 1 & 1 \\
		\hspace{-1mm} $x_6$ \hspace{-1mm} & 1 & 1 & 0 & 1 & 1 \\
		\hline
\end{tabular}
\begin{itemize}
    \item  $E_1 = \{(f_1,1),(f_2,1)\}$ 
    \item  $E_2 = \{(f_3,1),(f_4,1)\}$
    \item  $E_3 = \{(f_1,1)\}$
\end{itemize}
\end{multicols}
\end{nolinenumbers}

The decision $\kappa(x_5)=1$ has two AXp's (over the whole feature space $\mbb{F}$): 
$E_ 1$ and $E_2$. It has two d-AXp's:  $E_2$ and $E_3$.
The fact that $E_3$ is a d-AXp is a consequence of the fact that the pair 
$(f_1,1)$, $(f_2,0)$ never occurs in instances of $\fml{T}$. 
Note that $E_3$ is a d-wAXP while it is not a weak AXp. 
\end{example}

We have seen in Property~\ref{axp-complexity} that the problems of testing and finding 
abductive explanations are not tractable. We show next that their \textbf{sample-based versions are tractable}. Indeed, there is an obvious algorithm (by applying directly the definition) 
with complexity $O(mn)$ for testing whether a set of literals $E$ is a weak dataset-based AXp (d-wAXp), where $n$ and $m$ stand for the number of features and instances in the dataset $\fml{T}$ respectively. We can also test whether a weak d-AXp $E$ is subset-minimal in $O(mn^2)$ time by testing if $E$ remains a weak d-AXp after deletion of each literal. Indeed, as for AXp's~\cite{CP21}, a d-AXp can be found by starting with $E=x$, the instance to be explained, and in turn for each of the $n$ elements of $E$, delete it if $E$ remains a weak d-AXp after its deletion. It follows that a d-AXp can be found in $O(mn^2)$ time.

\begin{mytheorem}
Let  $n = |\F|$, $m = |\fml{T}|$ and $E \in \mbb{E}$. 
\begin{itemize}
	\item Testing whether $E$ is a d-wAXp can be achieved in $O(mn)$ time. 
	\item Finding a d-AXp can be achieved in $O(mn^2)$ time.
\end{itemize}
\end{mytheorem}
\subsection{Coverage-based explanations based on samples}

We now study the sample-based versions of the three types of explanations that take into account the coverage of explanations. Coverage is now among instances in the dataset $\mathcal{T} \subseteq \mbb{F}[\mathcal{C}]$.

\begin{definition}[d-CPI-Xp, d-mCPI-Xp, d-pCPI-Xp]\label{def:d-CPI}
Let $x \in  \mbb{F}[\C]$. 
A \emph{dataset-based CPI-explanation} (d-CPI-Xp) of $\kappa(x)$ is a partial assignment 
$E \in \mbb{E}$ such that:
\begin{itemize}
    \item $E(x)$,  
    \item $\forall y \in \fml{T}. (E(y) \rightarrow (\kappa(y)=\kappa(x)))$, 
    \item $\nexists E' \in \mbb{E}$ such that $E'$ satisfies the above conditions and 
    strictly subsumes $E$ in $\fml{T}$.  
\end{itemize}
A \emph{dataset-based minimal CPI-explanation} (d-mCPI-Xp) of $\kappa(x)$ is a subset minimal d-CPI-Xp of $\kappa(x)$. 
A \emph{dataset-based preferred CPI-explanation} (d-pCPI-Xp) of $\kappa(x)$ is a representative of the set of d-mCPI-Xp's of $\kappa(x)$ in $\fml{T}$.
\end{definition}

\begin{proposition}\label{links-d-exp}
The following inclusions hold: 
\begin{itemize}
 	\item $\bF_{d-CPI-Xp}(x) \subseteq \bF_{d-wAXpc}(x)$
	\item $\bF_{d-mCPI-Xp}(x) \subseteq \bF_{d-AXpc}(x)$
	\item $\bF_{d-pCPI-Xp}(x) \! \subseteq \! \bF_{d-mCPI-Xp}(x) \! \subseteq \! \bF_{d-CPI-Xp}(x)$
\end{itemize}
\end{proposition}

Even when the set of constraints $\mathcal{C}$ is empty, dataset-based AXp's do not coincide with dataset-based 
CPI-Xp's or mCPI-Xp's. This is mainly due to the notion of subsumption which privileges 
explanations with greater coverage.  

\begin{exs}{ex0bis}
Consider again the theory below and recall that for $x \in \mbb{F}$, 
$\kappa_2(x) = \neg f_1$. Suppose $\C = \emptyset$ and let us focus on the sample 
$\fml{T} = \{x_1, x_2, x_3\}$ ($x_4$ being discarded). 

\begin{nolinenumbers}
\begin{multicols}{2}
	\begin{center}
		\begin{tabular}{|c|cc|c|}
			\hline
			&$f_1$ & $f_2$ & $\kappa_2(x_i)$ \\
			\hline 
            $x_1$ & 0 & 0 &  1 \\
            \rowcolor{maroon!20} $x_2$ & 0 & 1 &  1 \\
			$x_3$ & 1 & 0 &  0 \\
           \rowcolor{gray!20} $x_4$ & 1 & 1 &  0 \\
			\hline
		\end{tabular}
	\end{center}

\begin{itemize}
	\item $E_1 = \{(f_1,0)\}$ 
	\item $E_2 = \{(f_2,1)\}$ 
	\item $E_3 = \{(f_1,0),(f_2,1)\}$
\end{itemize}	
\end{multicols}
\end{nolinenumbers}

\noindent The decision $\kappa_2(x_2)$ has three d-wAXps ($E_1, E_2, E_3$) and two 
d-AXp's ($E_1, E_2$). However, it has a single d-CPI-Xp/d-mCPI-Xp:  $E_1$. 
Indeed, its coverage in $\fml{T}$ is $\{x_1, x_2\}$, which is a super-set of 
the coverage $\{x_2\}$ of $E_2, E_3$.
\end{exs}


We show that considering coverage in the definition of prime implicant does not greatly increase the complexity 
of finding explanations based on the dataset. There is a polynomial-time algorithm for testing whether a 
partial assignment $E$ is a d-CPI-Xp and indeed for finding a d-CPI-Xp.
Furthermore, finding a subset-minimal d-CPI-Xp is asymptotically no more costly than finding a d-CPI-Xp. 

\begin{mytheorem} \label{prop:testdCPIXp}
Let $E \in \mbb{E}$. 
\begin{itemize}
	\item Testing whether $E$ is a d-CPI-Xp can be achieved in $O(m^2n)$ time.
	\item Finding a d-CPI-Xp, a minimal d-CPI-Xp and a preferred d-CPI-Xp can be achieved in $O(m^2n^2)$ time.
\end{itemize} 
\end{mytheorem}

Table~\ref{tab:abductive} summarizes all the complexity results concerning the different types of AXps' reviewed so far. 

\begin{table} 
\centering
\begin{tabular}{|c|c|c|} 
\hline
 & Complexity & Complexity \\ 
Explanation &  of testing  &  of finding one \\ 
\hline \hline
d-wAXp& P  &   polytime       \\
d-AXp & P  & polytime \\ 
\hline
d-CPI-Xp  & P & polytime \\ 
d-mCPI-Xp & P & polytime \\ 
d-pCPI-Xp & P & polytime \\
\hline \hline
wAXp   & co-NP-complete  &  polytime       \\ 
AXp    & $\text{P}^{\text{NP}}$ & $\text{FP}^{\text{NP}}$ \\ 
\hline
CPI-Xp & $\Pi_2^P$-complete & $\text{FP}^{\Sigma_2^{\text{P}}}$ \\ 
mCPI-Xp & $\Pi_2^P$-complete & $\text{FP}^{\Sigma_2^{\text{P}}}$ \\ 
pCPI-Xp & $\Pi_2^P$-complete & $\text{FP}^{\Sigma_2^{\text{P}}}$ \\ 
\hline
\end{tabular} 

\vspace{0.2cm}

\caption{Complexities of testing/finding different 
explanations. 
$\text{FP}^{\mathcal{L}}$ is the class of function problems
that can be solved by a polynomial number of calls to an  oracle for the language $\mathcal{L}$.
We assume a white box, i.e. $\kappa$ is an arbitrary but \emph{known} function,
except for the case of sample-based explanations (where $\kappa$ may be a black-box function).}  
\label{tab:abductive}
\end{table}

\textbf{To sum up}, the previous definitions and results show that 
the sample-based approach presents three advantages: 
i) testing the validity of an explanation is linear in the size of the sample whatever the function $\kappa$ and the constraints, 
ii) it can be applied even when the classifier is a black-box, 
iii) sample-based abductive explanations may be smaller and hence easier to interpret for a human user. 
\begin{table*}
	{\small 
		\centering
		\begin{tabular}{|l|c|c|c!{\color{maroon!20}\setlength{\arrayrulewidth}{1pt}\vline}c|a|b!{\color{maroon!20}\setlength{\arrayrulewidth}{1pt}\vline}c|c|c!{\color{maroon!20}\setlength{\arrayrulewidth}{1pt}\vline}c|c|}
			\hline
			 & \hspace{-1mm} wAXp  \hspace{-1mm}   &  \hspace{-1mm} AXp \hspace{-1mm} & \hspace{-1mm} AXpc \hspace{-1mm}  &  \hspace{-1mm} CPI-Xp \hspace{-1mm} &  \hspace{-1mm} mCPI-Xp \hspace{-1mm}  & \hspace{-1mm} pCPI-Xp \hspace{-1mm} &  \hspace{-1mm} dCPI-Xp \hspace{-1mm}  & \hspace{-1mm} dmCPI-Xp \hspace{-1mm}  & \hspace{-1mm}  dpCPI-Xp \hspace{-1mm} & \hspace{-1mm}  d-wAXp  \hspace{-1mm} & \hspace{-1mm} d-AXp \hspace{-1mm} \\
			\hline \hline
			Success           & $\checkmark$ & $\checkmark$ & $\checkmark$ &  $\checkmark$ & $\checkmark$ & $\checkmark$ & $\checkmark$     & $\checkmark$  & $\checkmark$ & $\checkmark$  & $\checkmark$    \\
			\hline
			Non-Triv.    & $\checkmark$ & $\checkmark$ & $\checkmark$ &  $\checkmark$ & $\checkmark$ & $\checkmark$ & $\checkmark$     & $\checkmark$  & $\checkmark$ & $\checkmark$  & $\checkmark$   \\
			\hline
			Irreduc.   &    $\times$   & $\checkmark$ & $\checkmark$  &   $\times$   & $\checkmark$ & $\checkmark$ &  $\times$   & $\checkmark$  & $\checkmark$ & $\times$  & $\checkmark$   \\
			\hline
			Coherence         &  $\checkmark$ & $\checkmark$ & $\checkmark$ &  $\checkmark$ & $\checkmark$ & $\checkmark$ &   $\times$         &    $\times$        &   $\times$        & $\times$ & $\times$  \\
			\hline\hline
			Consist.       & $\checkmark$ & $\checkmark$ & $\checkmark$  & $\checkmark$ & $\checkmark$ & $\checkmark$ &  $\checkmark$    &  $\checkmark$ & $\checkmark$ & $\checkmark$ & $\checkmark$  \\
			\hline
			Indep.      &    $\times$  &  $\times$   &  $\times$    &  $\checkmark$  & $\checkmark$ & $\checkmark$ & $\times$     &  $\times$ & $\checkmark$ & $\times$ & $\times$ \\         
			\hline
			Non-Equiv.    &    $\checkmark$       &  $\checkmark$  &    $\times$   &     $\times$        &   $\times$        & $\checkmark$ &      $\times$         &     $\times$       & $\checkmark$ & $\times$ & $\times$ \\ 
			\hline
		\end{tabular} \vspace{0.2cm}
		
		\caption{The symbol $\checkmark$ (resp. $\times$) stands for satisfied (resp. violated).}
		
		\label{tab}
	}
\end{table*}
\section{Properties of explanation functions}\label{sec:axioms}

We have seen that for each type of explanation studied in this paper,
there is a corresponding explanation function $\bF$ 
mapping every instance (in $\mbb{F}$ or $\mbb{F}[\C]$) into a subset of $\mbb{E}$. 
In this section, we provide seven desirable properties for an explanation function. 
The first four properties have counterparts in \cite{Amgoud22}, where explanation 
functions explain the \textit{global} behaviour of a classifier in a non-constrained setting, and  so answer the question: ``why does a classifier recommend a given class in general?'' 
We adapt the properties for explaining individual decisions and 
introduce three novel ones that concern how a function should deal with constraints. 

\begin{definition}\label{axioms}
Let $\bF$ be an explanation function.
\begin{description}
	\item (\textit{Success}) $\ \forall x \in \mbb{F}[\C]$, $\bF(x) \not= \emptyset$.  
	\item (\textit{Non-Triviality}) $\ \forall x \in \mbb{F}[\C]$, $\forall E \in \bF(x)$, $E \not= \emptyset$.
	\item (\textit{Irreducibility}) $\ \forall x \in \mbb{F}[\C]$, $\forall E \in \bF(x)$, $\forall l \in E$, 
	$\exists x' \in \mbb{F}[\C]$ such that $\kappa(x') \neq \kappa(x)$ and $(E \setminus \{ l \})(x')$. 
	\item (\textit{Coherence}) $\ \forall x, x' \in \mbb{F}[\C]$ such that $\kappa(x) \neq \kappa(x')$, $\forall E \in \bF(x)$, $\forall E' \in \bF(x')$, $\nexists x'' \in \mbb{F}[\C]$ s.t. $(E \cup E')(x'')$.

	\item (\textit{Consistency}) $\ \forall x \in \mbb{F}[\C]$, $\forall E \in \bF(x)$, $\C(E)$ holds.
	\item (\textit{Independence}) $\ \forall x \in \mbb{F}[\C]$,  $\nexists E, E' \in \bF(x)$ such that  
	      $E \to E' \in \C^*$ and $E' \to E \notin \C^*$.
    \item (\textit{Non-Equivalence}) $\ \forall x \in \mbb{F}[\C]$, $\forall E$, $E'$ $\in \bF(x)$, $E$ $\not\approx$ $E'$.	      
\end{description}
\end{definition}

\textit{Success} ensures existence of explanations. \textit{Non-Triviality} discards empty explanations as 
they are non-informative. 
\textit{Irreducibility} states that an explanation should not contain unnecessary information.
\textit{Coherence} ensures a global compatibility of the explanations provided for all the instances. It avoids erroneous explanations. Consider a function $\kappa$ which classifies animals as mammals or not, where animals 
are described by $n$ features such as \emph{milks its young}, \emph{lays eggs}, etc. Let $x$ be a mouse and $x'$ 
an eagle. If the explanation $E$ for $\kappa(x)=1$ is that mice milk their young and the explanation $E'$
for $\kappa(x')=0$ is that eagles lay eggs, then coherence is not satisfied because there are animals $x''$ (such as the platypus) which milk their young and lay eggs. 
\textit{Consistency} ensures that explanations satisfy all constraints in $\C$. 
\textit{Independence} ensures that dependency constraints are considered and the explanations of a decision 
should be pairwise independent. 
\textit{Non-equivalence} avoids equivalent explanations. This property is important for reducing the number of explanations. 

\vspace{0.15cm}

We show that the seven properties are \textit{compatible}, i.e., they can be satisfied all together by an explanation function. 

\begin{mytheorem}\label{links+com}
The properties are compatible. 
\end{mytheorem}


Table~\ref{tab} summarizes the properties satisfied by each type of explanation discussed in the paper. 
It thus provides a comprehensive formal comparison of their explainers, and sheds light on the key properties 
that distinguish any pair of explainers. 

\begin{mytheorem}
	The properties of Table~\ref{tab} hold.
\end{mytheorem}

The results confirm that existing definitions of abductive explanations ignore dependency 
constraints (they violate independence). However, they satisfy in a vacuous way consistency 
since the latter deals only with feasible instances (elements of $\mbb{F}[\C]$).   
The three novel types that are generated from the feature space handle properly integrity 
constraints (satisfaction of consistency) and dependency constraints (satisfaction of independence). 

The results show also that among the new types, only the two versions of preferred CPI-Xp satisfy non-equivalence. 
Hence, they use the most discriminatory selection criterion. 
Non-equivalence is surprisingly satisfied by wAXp and AXp because they consider the entire feature space, 
and in this particular case, two different partial assignments can never have the same coverage. 
Note that the property is violated by their sample-based versions.  

Another result which is due to the use of the whole feature space concerns the satisfaction of coherence. 
The property is lost 
when explanations are based on a dataset, thus erroneous explanations may be provided for decisions. The main reason behind this issue is that we generate explanations under incomplete information, and thus some explanations may not hold when tested on the whole feature space. Another consequence of incompleteness of information is that the sample-based versions of CPI-Xp and mCPI-Xp violate independence due to missing instances in the sample. 

To sump up, the explainer that generates preferred CPI-Xp is the only one that satisfies all the properties.  

\section{Related work}\label{sec:related-w}

As explained in the introduction, abductive explanations have largely been studied in the XAI literature. 
However, to the best of our knowledge only a few works (\cite{CP21,GorjiAAAI22}) have considered the constrained 
setting. We have shown that those works deal only partially with constraints as they ignore dependency constraints.

Many papers \cite{BaileyS05,BandaSW03,Junker04,LiffitonS05} are concerned with explaining the inconsistency of constraints, which is quite far from the problem we are studying (explaining 
the output of a classifier $\kappa$ in a constrained feature-space).  \cite{IgnatievS21} add 
new constraints in a SAT solver in order to quickly find explanations of a classifier, 
again a different problem. 

\section{Conclusion}\label{conclusion}

Our work is the first to wholly take into account constraints for producing abductive explanations. 
A general conclusion that can be drawn from our study is that constraints may lead to 
less-redundant explanations, but at the cost of a potential increase in complexity. 
Another conclusion is that sample-based versions of explanations provide a tractable alternative, 
especially in the case of black-box classifiers. The downside of the approach is,
unfortunately, explanations are only valid for the instances in the dataset
and not for the whole feature space. Therefore, there is a trade-off between computational complexity 
and coherence of explanations.	


This work can be extended in different ways. We plan to characterize the whole family of explainers that 
satisfy all (or a subset of) the properties. 
We also plan to define sample-based 
explainers that consider constraints and satisfy the coherence property. 
Another avenue of future research is to learn constraints from the dataset:
we would limit ourselves to small-arity constraints to make this feasible.  

Recall that the set of all pCPI-Xp's contains a single
 representative from each equivalence class, where
 two explanations are equivalent if they cover the same set of instances.
 A challenging open problem is the enumeration of pCPI-Xp's,
 which corresponds to enumerating
 abductive explanations whose coverages are all 
 pairwise incomparable for subset inclusion.

\section*{Acknowledgments}
\vspace{-1mm}
This work was supported by the AI Interdisciplinary Institute ANITI, funded by the French
program “Investing for the Future – PIA3” under grant agreement no. ANR-19-PI3A-0004.


\begin{appendix}

\section{Appendix: Proofs}

\setcounter{prop}{0}
\setcounter{mytheorem}{0}
\setcounter{lem}{0}

\begin{proposition} 
Let $x \in \mbb{F}$ be such that $\C(x)$, and $E \in \mbb{E}$.  
If $E$ is an AXp of $\kappa(x)$, then $\exists E' \in \mbb{E}$
such that $E' \subseteq E$ and $E'$ is an AXpc of $\kappa(x)$.   
\end{proposition}

\begin{proof}
Let  $x \in \mbb{F}$ and $E \in \mbb{E}$.  Assume that $E$ is an AXp of $\kappa(x)$. 
Then, $E(x)$ and $\forall y \in \mbb{F}. \big(E(y) \rightarrow (\kappa(y)=\kappa(x))\big)$. 
Let $T = \{y \in \mbb{F} \ | \ E(y) \wedge \big(\kappa(y)=\kappa(x)\big)\}$ and 
$T' = \{y \in \mbb{F} \ | \ \C(y) \wedge E(y) \wedge \big(\kappa(y)=\kappa(x)\big)\}$. 
Clearly, $T' \subseteq T$. Hence, $\exists E' \subseteq E$ which 
verifies the conditions of an AXpc. 
\end{proof}	
\begin{proposition} 
Let $X \subseteq \mbb{F}$ and $E, E' \in \mbb{E}$. 
\begin{itemize}
\item The following statements are equivalent. 
\begin{itemize}
	\item $E'$ subsumes $E$ in $X$.
	\item $\cov_X(E) \subseteq \cov_X(E')$.
\end{itemize}
\item If $E' \subset E$, then $E'$ subsumes $E$ in $X$. The converse does not always hold. 
\item If $E \neq E'$, then $\cov_\mbb{F}(E) \neq \cov_\mbb{F}(E')$.
\end{itemize}
\end{proposition}

\begin{proof}
Let $X \subseteq \mbb{F}[\C]$ and $E, E' \in \mbb{E}$.  
Assume $E'$ subsumes $E$ in $X$. Then, $\forall x \in X. (E(x) \rightarrow E'(x))$. 
Thus, $\{x \in X \ | \ (E(x)\} \subseteq \{x \in X \ | \ (E'(x)\}$,
and so $\cov_X(E) \subseteq \cov_X(E')$. 
Assume now that $\cov_X(E) \subseteq \cov_X(E')$. This means 
$\{x \in X \mid (E(x)\} \subseteq \{x \in X \ | \ (E'(x)\}$, and 
so $E'$ subsumes $E$ in $X$.
The second property (when $E' \subset E$) follows from the inclusion  $\cov_X(E') \subseteq \cov_X(E)$ and the 
above equivalence. 

Assume $E \neq E'$.
Since $\mbb{F}$ contains all the possible assignments of values and $E, E'$ are  \textbf{consistent} assignments (i.e., they do not contain 
two literals assigning different values to the same feature), 
and knowing that domains of all features are of size at least 2,
we can deduce that
$\exists z, z' \in \mbb{F}$ such that $E \subseteq z$ and $E' \not\subseteq z$ and 
 $E' \subseteq z'$ and $E \not\subseteq z'$. 
Thus, $\cov(E) \not\subseteq \cov(E')$ and $\cov(E') \not\subseteq \cov(E)$. 
So, $\cov(E) \neq \cov(E')$.
\end{proof}
\begin{proposition} 
Let $E, E' \in \mbb{E}$. If $E \to E' \in \C^*$, then $\forall X \subseteq \mbb{F}[\C]$, 
$E'$ subsumes $E$ in $X$. 
\end{proposition}

\begin{proof}
Let $E, E' \in \mbb{E}$ such that $E \to E' \in \C^*$. 
From the definition of dependency constraints, $\forall x \in \mbb{F}[\C]$, if $E(x)$ then $E'(x)$. 
Let $X \subseteq \mbb{F}[\C]$. It follows that $\forall x \in X$, if $E(x)$ then $E'(x)$. 
So, $E'$ subsumes $E$ in $X$. 
\end{proof}
\begin{proposition} 
Let $x \in \mbb{F}[\C]$. The following inclusions hold. 
\begin{enumerate}
	\item $\bF_{AXp}(x) \subseteq \bF_{wAXp}(x)$,
	\item $\bF_{CPI-Xp}(x) \subseteq \bF_{wAXpc}(x)$,
	\item $\bF_{mCPI-Xp}(x) \subseteq \bF_{AXpc}(x) \subseteq \bF_{wAXpc}(x)$,
	\item $\bF_{pCPI-Xp}(x) \subseteq \bF_{mCPI-Xp}(x) \subseteq \bF_{CPI-Xp}(x)$, 
	\item If $\C = \emptyset$, then 
	the following hold:
	\begin{enumerate}
	      \item $\bF_{AXp}(x) = \bF_{AXpc}(x)$ 
	      \item $\bF_{AXpc}(x) = \bF_{mCPI-Xp}(x)$.
    \end{enumerate}
\end{enumerate}
\end{proposition}

\begin{proof}
Let $x \in \mbb{F}[\C]$. 
The inclusions 1), 4) and $\bF_{AXpc}(x) \subseteq \bF_{wAXpc}(x)$ follow straightforwardly from the definitions of the 
corresponding types of explanations. 

Let us show the inclusion $\bF_{CPI-Xp}(x) \subseteq \bF_{wAXpc}(x)$. 
Assume $E \in \mbb{E}$ is a CPI-Xp for $\kappa(x)$. Then the following properties hold: 

\begin{description}
\item [a)] $E(x)$, 
\item [b)] $\forall y \in \mbb{F}[\C]. (E(y) \rightarrow (\kappa(y)=\kappa(x)))$, 
\item [c)] $\nexists E' \in \mbb{E}$ such that $E'$ 
           satisfies the conditions a) and b) and strictly subsumes $E$ in $\mbb{F}[\C]$. 
\end{description}

From b), it follows that $\forall y \in \mbb{F}. (\C(y) \wedge E(y) \rightarrow (\kappa(y)=\kappa(x)))$. 
So, $E$ is a weak AXpc for $\kappa(x)$.

Let us show the inclusion $\bF_{mCPI-Xp}(x) \subseteq \bF_{AXpc}(x)$.
Assume $E \in \mbb{E}$ is an mCPI-Xp for $\kappa(x)$. Then the following properties hold: 
\begin{description}
\item [a)] $E(x)$, 
\item [b)] $\forall y \in \mbb{F}[\C]. (E(y) \rightarrow (\kappa(y)=\kappa(x)))$, 
\item [c)] $\nexists E' \in \mbb{E}$ such that $E'$ 
           satisfies the conditions a) and b) and strictly subsumes $E$ in $\mbb{F}[\C]$. 
\item [d)] $\nexists E' \subset E$ which satisfies a), b) and c). 
\end{description}
Note that the condition b) is equivalent to the following formula: 
$\forall y \in \mbb{F}. (\C(y) \wedge E(y) \rightarrow (\kappa(y)=\kappa(x)))$. 
Assume that $E \notin \bF_{AXpc}(x)$. Since $E$ satisfies a) and b), then 
$\exists E' \subset E$ which satisfies the same properties. 
From Proposition~\ref{cov}, $\cov(E) \subseteq \cov(E')$. 
From condition c), $E'$ does not strictly subsume $E$, so  
$\cov(E) = \cov(E')$. Hence, $E'$  satisfies a), b), c) and 
this contradicts condition d) of $E$.

Assume that $\C = \emptyset$, then $\mbb{F}[\C] = \mbb{F}$. 
The equality  $\bF_{AXp}(x) = \bF_{AXpc}(x)$ 
follows from the fact $\forall y \in \mbb{F}$, $\C(y) \equiv \top$ (by 
assumption), and thus 
$\forall y \in \mbb{F}. (\C(y) \wedge E(y) \rightarrow (\kappa(y)=\kappa(x)))$ is 
equivalent to 
$\forall y \in \mbb{F}. (E(y) \rightarrow (\kappa(y)=\kappa(x)))$. 

We already have the inclusion $\bF_{mCPI-Xp}(x) \subseteq \bF_{AXpc}(x)$
from property 3.
Let us show the inclusion $\bF_{AXpc}(x) \subseteq \bF_{mCPI-Xp}(x)$. 
Let $E \in \bF_{AXpc}(x)$, then from the property 5(a) above, 
$E \in \bF_{AXp}(x)$ and thus it satisfies the following properties:
\begin{description}
\item [a)] $E(x)$, 
\item [b)] $\forall y \in \mbb{F}. (E(y) \rightarrow (\kappa(y)=\kappa(x)))$, 
\item [c)] $\nexists E' \subset E$ which satisfies a) and b). 
\end{description}
Assume $E \notin \bF_{mCPI-Xp}(x)$. Then, $\exists E' \in \mbb{E}$ such that 
$E'$ satisfies a) and b) and strictly subsumes $E$ in . From Proposition~\ref{cov}, 
$\cov(E) \subset \cov(E')$ \textbf{(A)}. Since $\mbb{F}[\C] = \mbb{F}$ and $E' \neq E$, from 
the 3rd property of Proposition~\ref{cov}, $\cov(E) \not\subseteq \cov(E')$ 
which contradicts assumption \textbf{(A)}.
%
\end{proof}

Before studying the computational problem of deciding whether an 
explanation is a coverage-based PI-explanation (CPI-Xp), we introduce 
some notation. Recall that a partial assignment $A$
can be written as a predicate, with $A(x)$ meaning that the 
assignment $x$ is an extension of $A$.
For partial assignments $A,B$ we use the notation $A \imp B$ for 
$\forall x \in \mbb{F}[\C]. (A(x) \rightarrow B(x))$.
In the following we suppose that $kappa(v)$ is the decision to explain.
Let $L_{v,E,\mathcal{C}}$ denote the set of literals in 
$v$ implied by $E \land \C$. Thus
$E \imp L_{v,E,\mathcal{C}}$,
where we view the set of literals $L_{v,E,\mathcal{C}}$ as a partial assignment.

\begin{mytheorem} 
The problem of testing whether a weak AXpc $E$ is a coverage-based PI-explanation 
is $\Pi_2^{\rm P}$-complete.
\end{mytheorem}

\begin{proof}
Recall that the set of possible explanations $\mbb{E}$ 
of $\kappa(v)$ is the set of all partial assignments 
that are subsets of $v$ (viewed as a set of $n$ literals).
A weak AXpc $E \in \mbb{E}$ is a coverage-based PI-explanation of $\kappa(v)$ iff

\setcounter{equation}{0}
\begin{align}
\forall E' \subseteq L_{v,E,\mathcal{C}}. 
(\exists y \in \mbb{F}[\C]. ( E'(y) \land \neg \kappa(y)=c ) \ \lor \nonumber \\ 
\ \forall x \in \mbb{F}[\C]. ( E'(x) \rightarrow E(x)))  
\label{eq:compactnessApp}
\end{align}

By rewriting the definition of whether a given weak AXpc $E$ is a coverage-based 
PI-explanation in equation (\ref{eq:compactnessApp}) as
\begin{align}
\forall E' \subseteq L_{v,E,\mathcal{C}}. \forall x \in \mbb{F}[\C].  
\exists y \in \mbb{F}[\C].  ( \nonumber \\ \ ( E'(y) \land \neg \kappa(y)=c ) \ \lor \  E'(x) \rightarrow E(x) \ )  \label{eq:compactness2}
\end{align}
it is clear that this problem belongs to $\Pi_2^{\rm P}$.

To show completeness, we give a reduction from $\forall\exists$SAT. Let $\phi(w,z)$ be a boolean formula
on boolean vectors $w,z$ of length $n$.
We will construct constraints $\C$, a classfier $\kappa$, a vector v and a weak AXpc $E$ of the decision
$\kappa(v)=1$ such that $E$ is a coverage-based PI-explanation iff $\forall w \exists z \phi(w,z)$ is satisfiable.

Let $v$ be the vector $(1,\ldots,1)$ of length $4n+1$ and $E$ the singleton $\{ (f_{4n+1},1) \}$.
The set $\C$ contains the constraints 
\[  f_{4n+1} \ \rightarrow \ f_{2k+1} \quad (0 \leq k \leq 2n-1)
\]
which means that any vector satisfying $E(x)$ and $\C$ must have all odd-numbered features equal to 1.
It follows that $L_{v,E,\mathcal{C}} = \{(f_{2k+1},1) \mid 0 \leq k \leq 2n\}$.
In a feature vector $x=(f_1,\ldots,f_{4n+1})$, 
we call the four features $f_{4j+1}, f_{4j+2}, f_{4j+3}, f_{4j+4}$ 
the $j$th block ($j=0,\ldots,n-1$).
The set $\C$ also contains constraints which impose 
that when $f_{4n+1} \neq 1$, in each block
the only permitted assignments to the four features are $(1,0,0,0)$, $(1,1,0,1)$,
$(0,0,1,0)$ and $(0,1,1,1)$ (i.e. $f_{4j+3} \neq f_{4j+1}$ and 
$f_{4j+4} = f_{4j+2}$). 

There are no constraints in $\C$ which impose $E(x)$ (i.e. that $f_{4n+1}=1$)
given the values of the first $4n$ features.
This means that in equation (\ref{eq:compactnessApp}) the expression
$\forall x \in \mbb{F}[\C]. (E'(x) \rightarrow E(x))$ holds 
iff $E=\{ (f_{4n+1},1) \} \subseteq E'$
or for some $i \in \{0,\ldots,n-1\}$ we have 
$\{ (f_{4i+1},1), (f_{4i+3},1) \} \subseteq E'$ 
(which would imply $\forall x \in \mbb{F}[\C] \neg E'(x)$ due to the above constraints on each block). 

Consider an $E'_1$ such that for some $i \in \{0,\ldots,n-1\}$ we have 
$\{ (f_{4i+1},1), f_{4i+3},1) \} \cap E'_1 = \emptyset$,
and let $E'_2 = E'_1 \cup \{ (f_{4i+1},1) \}$. 
Then in the quantification over $E'$ in equation (\ref{eq:compactnessApp}),
the case $E'=E'_1$ will be subsumed by the case $E'=E'_2$, since $E'_2(y) \rightarrow E'_1(y)$,
and hence the case $E'=E'_1$ can be ignored in the quantification.

We write $valid(E')$ to represent the property that 
$\{ (f_{{4n+1}},1) \} \notin E'$ and for each $i=0,\ldots,n-1$, exactly one 
of $(f_{4i+1},1)$ and $(f_{4i+3},1)$ belongs to $E'$.
By the above discussion, equation (\ref{eq:compactnessApp}) is equivalent to
\begin{align}
\forall E' \subseteq L_{v,E,\mathcal{C}}. ( \ valid(E') \ \rightarrow \ \nonumber \\ (
\exists y \in \mbb{F}[\C]. ( E'(y) \land \neg \kappa(y)=1 )) \ )  \label{eq:compactness3}
\end{align}

Given a feature vector $y$, define boolean vectors $w$ and $z$ as follows:
$w_i = y_{4i+1}$ and $z_i = y_{4i+2}$ ($i=0,\ldots,n-1$).
Thus, in the $i$th block, the four possible assignments 
$(1,0,0,0)$, $(1,1,0,1)$, $(0,0,1,0)$, $(0,1,1,1)$ to 
$(y_{4j+1}, y_{4j+2}, y_{4j+3}, y_{4j+4})$
correspond respectively to the four possible assignments $(1,0)$, $(1,1)$, $(0,0)$, $(0,1)$ to $(w_i,z_i)$.
Recall that $L_{v,E,\mathcal{C}}$ corresponds to the assignments $(f_{2i+1},1)$ to odd-numbered features.
Thus, the quantification over valid $E' \subseteq L_{v,E,\mathcal{C}}$ in equation (\ref{eq:compactness3})
corresponds to the quantification over all $w$: 
the choice of $E'$ determines which one of the two 
assignments $(f_{4i+1},1)$ or $(f_{4i+3},1)$ belongs to $E'$ 
which in turn forces $y_{4i+1}=1$ or $y_{4i+3}=1$
for $y$ satisfying $E'(y)$, which in turn corresponds to $w_i=1$ or $w_i=0$.
If $(f_{4i+1},1) \in E'$ (respectively $(f_{4i+3},1) \in E'$) 
then $y \in \mbb{F}[\C] \land E'(y)$ holds
only if the $i$th block of $y$ is $(1,0,0,0)$ or $(1,1,0,1)$ (respectively, $(0,0,1,0)$ or $(0,1,1,1)$).
Thus, for a given $E'$, the quantification
$\exists y$  in equation (\ref{eq:compactness3}) for $y$ satisfying the constraints $\C$ and $E'(y)$
is a quantification over the even-numbered features $y_{4i+2}$, $y_{4i+4}$ (which are necessarily equal due to
the constraints $\C$) and hence over $z$ (since $z_i=y_{4i+2}$ for $i=0,\ldots,n-1$). Now define
\[ \kappa(y)  =  E(y) \lor \neg \phi(w(y),z(y)) \ = \  y_{4n+1} \
 \lor \ \neg \phi(w(y),z(y))
\]
where $w(y)_i=y_{4i+1}$, $z(y)_i=y_{4i+2}$ ($i=0,\ldots,n-1$).
Clearly $E$ is a weak AXpc of $\kappa(v)=1$, where $v=(1,\ldots,1)$, 
since $E(v)$ is true.
Testing whether $E$ is a coverage-based PI-explanation amounts to testing whether equation (\ref{eq:compactness3}) holds
which is equivalent to
\begin{align}
\forall E' \subseteq L_{v,E,\mathcal{C}}. ( \ valid(E') \rightarrow \nonumber \\ \exists y \in \mbb{F}[\C]. ( E'(y) \land 
\neg E(y) \land \phi(w(y),  z(y)) ) \  )  \label{eq:compactness4}
\end{align}
We have seen above that there is a one-to-one correspondence between valid $E' \subseteq L_{v,E,\mathcal{C}}$
and vectors $w=(w_1,\ldots,w_n)$ and that for a given $w$, the vectors $y$ that satisfy $y \in \mbb{F}[\C]
\land E'(y) \land \neg E(y)$ are in one-to-one correspondence with vectors $z=(z_1,\ldots,z_n)$.
Note that the condition $\neg E(y)$ simply imposes $y_{4n+1}=0$.
Thus equation (\ref{eq:compactness4}) is equivalent to
\[ \forall w \exists z \ \phi(w,z)
\]
which completes the polynomial reduction from $\forall\exists$SAT.
\end{proof}

We now give an algorithm to return one coverage-based PI-explanation of a given decision $\kappa(v)=c$.
It is based on the following idea: if a weak AXpc $E$ is not a coverage-based PI-explanation, then this is because there is a
weak AXpc $E'$ that strictly subsumes $E$. We call such an $E'$ a counter-example to the hypothesis that $E$
is a coverage-based PI-explanation. Therefore, starting from a weak AXpc $E$, we can look for a counter-example $E'$: if no 
counter-example exists then we return $E$, otherwise we can replace $E$ by $E'$ and 
re-iterate the process. This loop must necessarily halt since there cannot be an infinite sequence
of partial assignments $E_1,E_2,\ldots$ such that $E_{i+1}$ strictly subsumes $E_i$ ($i=1,2,\ldots$).

Recall that explanations are sets of literals $E \subseteq v$, where the assignment $v$ is viewed as the
set of literals $\{(f_i,v_i) \mid i=1,\ldots,n\}$. We suppose that the subprogram
satisfiable($\mathcal{C}$) determines the satisfiability of the set of constraints $\mathcal{C}$.
This allows us to write generic functions to test whether $A$ is a weak AXpc of $\kappa(v)=c$
and also whether we have $A \imp B$:
\begin{tabbing}
\quad \=function \=  weakAXpc($A,\mathcal{C},\kappa,c$): \quad \\
\>\>return not satisfiable($A \cup \mathcal{C} \cup \{\kappa(x) \neq c\}$) ; \\ \\
\>function implies($A,B,\mathcal{C}$): \quad \\ \>\>return not satisfiable($A \cup \mathcal{C} \cup \{\bigvee_{\ell \in B} \overline{\ell}\}$) ;
\end{tabbing}
In order to test recursively whether a weak AXpc $E$ is a coverage-based PI-explanation we require the following function 
which determines, given a weak AXpc $E$ and two sets $\mathit{infL}, supL$ of literals such that
$E \imp supL$, whether there is no counter-example $E'$ to the hypothesis that $E$ is a coverage-based PI-explanation such that
$\mathit{infL} \subseteq E' \subseteq supL$. It returns (true,-) if no such counter-example exists,
and (false,$E'$) where $E'$ is a counter-example, otherwise.
\begin{tabbing}
\quad \=function \= CPIexplanation($E,\mathit{infL},supL,\mathcal{C},\kappa,c$): \\
\quad \quad \=if not weakAXpc($supL,\mathcal{C},\kappa,c$) then return (true,-);  \\ \>\>$\{$no subset of $supL$ can be a counter-example$\}$ \\
\>if not implies($supL,E,\mathcal{C}$) then return (false,$supL$); \\ 
\>\>$\{$$supL$ is a counter-example$\}$ \\
\>for \=each literal $\ell$ in $\in supL \setminus \mathit{infL}$: \\
\>\>$\{$test if there is a counter-example without using $\ell$$\}$ \\
\>\>($res, \mathit{counterE}$) := \\ \>\> \quad \=CPIexplanation($E, \mathit{infL}, supL \setminus \{\ell\}, \mathcal{C}, \kappa, c$); \\
\>\> if $res =$ false then return ($res, \mathit{counterE}$); \\
\>\>\>$\{$have found a counter-example without using $\ell$$\}$ \\
\>\>$\mathit{infL}$ := $\mathit{infL} \cup \{\ell\}$; \\
\>\>\>$\{$add $\ell$ to list of literals that must be included \\
\>\>\> \quad in the counter-example$\}$; \\
\>return (true,-); $\{$since there are no more possible \\
\>\>\>counter-examples to try$\}$
\end{tabbing}
We can now give the algorithm to return a coverage-based PI-explanation of the decision $\kappa(v)=c$ under constraints $\mathcal{C}$.
\begin{tabbing}
\quad procedure findCPIexplanation($\mathcal{C}, \kappa, v, c$): \\
\quad \quad \=$E$ := $v$; $\{$initialise $E$ to the set of all literals in $v$$\}$ \\
\> $res$ := false; \\
\> while ($res$ = false) : \\
\> \quad \ \= $L$ := $E$; \\
\>\> for \=each literal $\ell \in v \setminus E$: \\
\>\>\> if implies($E,\{\ell\},\mathcal{C}$) then $L$ := $L \cup \{\ell\}$; \\
\>\> $\{ L$ has thus been initialised to $L_{v,E,\mathcal{C}}$, \\
\>\> \quad the set of literals in $v$ implied by $E \land \mathcal{C} \}$ \\
\>\> ($res$, $\mathit{counterE}$) := CPIexplanation($E,\emptyset,L,\mathcal{C},\kappa,c$); \\
\>\>\>$\{$test whether $E$ is a CPI-Xp$\}$ \\
\>\> if $res$ = false then $E$ := $\mathit{counterE}$; 
\ $\{$if $E$ is not \\ 
\>\>\>a CPI-Xp, replace it by the counter-example$\}$ \\
\> return $E$; \ $\{$$E$ is a CPI-Xp$\}$
\end{tabbing}

The algorithm halts since, as observed at the beginning of this
section, there cannot be an infinite sequence of strictly subsuming counter-examples. Furthermore,
the explanation $E$ which is returned is necessarily a coverage-based PI-explanation 
since the call CPIexplanation($E,\emptyset,L,\mathcal{C},\kappa,c$)
tests all possible counter-examples $E'$ for $\emptyset \subseteq E' \subseteq L_{v,E,\mathcal{C}}$.

It is worth pointing out that in the absence of constraints, the algorithm findCPIexplanation is equivalent to the
standard `deletion' algorithm for finding one AXp~\cite{CP21}. Indeed, if the subprogram
weakAXpc is polynomial-time and $\mathcal{C}=\emptyset$, then findCPIexplanation is also
polynomial-time. 

We have already observed that the number of iterations of the while loop in findCPIexplanation is bounded.
We can be more specific: the following proposition shows that the number of iterations is bounded by 
$n$, the number of features. We make the reasonable assumption that the classifier $\kappa$
is not a constant function, so the empty set cannot be a weak AXpc.
\begin{mytheorem} 
Let $n = |\F|$.
A CPI-Xp can be found by $n$ calls to an oracle for testing whether a given weak AXpc is a coverage-based PI-explanation.
\end{mytheorem}

\begin{proof}
It suffices to show that the number of iterations of the while loop in findCPIexplanation is at most $n$.
Let $(E_i,L_i)$ denote the parameters $(E,L)$ in the call CPIexplanation($E,\emptyset,L,\mathcal{C},\kappa,c$)
in the $i$th iteration of the while loop in fndCPIexplanation.
By construction, $L_i = L_{v,E_i,\mathcal{C}}$, so $E_i \imp L_i$.
Furthermore, $E_{i+1} \subseteq L_i$, since when CPIexplanation($E,\emptyset,L,\mathcal{C},\kappa,c$) returns
(false,$counterE$) we know that $counterE \subseteq L$. So, trivially, $L_i \imp E_{i+1}$. By transitivity, we can deduce that
$E_i \imp E_{i+1}$.
On the other hand, since $E_{i+1}$ is a counter-example to the hypothesis that $E_i$ is a
coverage-based PI-explanation, we know that
$\neg(E_{i+1} \imp E_i)$. Recalling the definition of 
$L_{v,E,\mathcal{C}} = \{ \ell \in v \mid E \imp \ell \}$, we necessarily have 
$L_{v,E_{i+1},\mathcal{C}} \subseteq L_{v,E_i,\mathcal{C}}$ since $E_i \imp E_{i+1}$.
Thus $L_{i+1} \subseteq L_i$. If we had $L_{i+1}=L_i$, then we would have
$E_{i+1} \imp L_{i+1}$, $L_{i+1}=L_i$ and $L_i \imp E_i$ (since $E_i \subseteq L_i= L_{v,E_i,\mathcal{C}}$),
which is in contradiction with $\neg(E_{i+1} \imp E_i)$. So it follows that $L_{i+1} \subset L_i$ ($i=1,2,\ldots$).
Since $L_1$ is the set of $n$ literals of the assignment $v$, the number of calls
CPIexplanation($E,\emptyset,L,\mathcal{C},\kappa,c$) in findCPIexplanation is thus bounded above by $n$
(since assuming $\kappa$ is not a constant function means that we never need to make this call with $L=\emptyset$).
\end{proof}

\begin{mytheorem} 
Let $n = |\F|$. 
A mCPI-Xp (resp. pCPI-Xp) can be found by $n$ calls to an oracle for testing whether a given weak AXpc is a coverage-based PI-explanation together with $2n$ calls to a SAT oracle.
\end{mytheorem}

\begin{proof}
Let $E$ be the CPI-explanation found by findCPIexplanation. By the definition of a CPI-explanation,
there is no weak AXpc $E' \subset E$ which strictly subsumes $E$. However, we can have a weak AXpc
$E' \subset E$ which is logically equivalent to $E$ in the sense that $E' \imp E$ and $E \imp E'$.
Since $E \imp E'$ is a consequence of $E' \subset E$, for a given $E' \subset E$, we only need to test
that $E'$ is a weak AXpc and that $E' \imp E$. As we have seen, each of these two tests can be achieved 
by a call to a SAT oracle. In order to find a minimal CPI-explanation from a CPI-explanation $E$, 
it suffices to use the following standard deletion algorithm for finding a minimal subset~\cite{ChenT95} 
satisfying a hereditary property.
\begin{tabbing}
\quad for \=each literal $\ell \in E$ : \\
\> $E'$ := $E \setminus \{\ell\}$ ; \\
\> if  weakAXpc($E'$,$\mathcal{C}$,$\kappa$,$c$) and implies($E'$,$E$,$\mathcal{C}$) \\
\>then $E$ := $E'$ 
\end{tabbing}
Finding a pCPI-Xp consists of finding one minimal CPI-Xp. 
\end{proof}
\begin{mytheorem} \label{thm:testdxAXp}
Let  $n = |\F|$, $m = |\fml{T}|$ and $E \in \mbb{E}$. 
\begin{itemize}
	\item Testing whether $E$ is a d-wAXp can be achieved in $O(mn)$ time. 

	\item Finding a d-AXp can be achieved in $O(mn^2)$ time.
\end{itemize}
\end{mytheorem}

\begin{proof}
There is an obvious algorithm for testing whether $E$ is a weak dataset-based AXp (d-wAXp). 
A d-AXp can be found by starting with $E=x$, the instance to be explained, and in turn for 
each of the $n$ elements of $E$, delete it if $E$ remains a weak d-AXp after its deletion. 
It follows that a d-AXp can be found in $O(mn^2)$ time.  

We can also test whether a weak d-AXp $E$ is subset-minimal in $O(mn^2)$ time by testing if $E$ remains a weak d-AXp after deletion of each literal. 
\end{proof}
\begin{proposition} 
The following inclusions hold: 
\begin{itemize}
 	\item $\bF_{d-CPI-Xp}(x) \subseteq \bF_{d-wAXpc}(x)$
	\item $\bF_{d-mCPI-Xp}(x) \subseteq \bF_{d-AXpc}(x)$
	\item $\bF_{d-pCPI-Xp}(x)$ $\subseteq$ $\bF_{d-mCPI-Xp}(x)$ $\subseteq$ $\bF_{d-CPI-Xp}(x)$
\end{itemize}
\end{proposition}

\begin{proof}
The first and second properties follow from the equivalence: 
$\forall y \in \fml{T}$, $E(y) \rightarrow (\kappa(y) = \kappa(x))$ $\iff$
$\forall y \in \fml{T}$, $(E(y) \wedge \C(y)) \rightarrow (\kappa(y) = \kappa(x))$, since all $y \in \fml{T}$ necessarily satisfy the constraints.
The third property follows straightforwardly from the definitions of the 3 types of explanations. 
\end{proof}
\noindent \textbf{Remark:} Recall the definition of the coverage of a set $E \in \mathcal{E}$:
\[ \text{cov}_{\mathcal{T}}(E) \ = \ \{ x \in \mathcal{T} \mid 
E(x) \land (\kappa(x)=c) \}
\]
For fixed $v$ and $\mathcal{T}$, recall that $L_{v,E,\mathcal{T}}$ is
the set of literals of $v$ implied by
the two conditions: (1) being equal to $v$ 
on features $E$ and (2) belonging to $\mathcal{T}$, i.e.
\begin{equation}  
L_{v,E,\mathcal{T}} \ =\ \{ \ell \mid \forall z \in \text{cov}_{\mathcal{T}}(E), \ell \in z \}
\label{eq:LvET}
\end{equation}
and for $y \in \mathcal{T}$, we define $S_{E,y}$ to be the set of 
literals in $L_{v,E,\mathcal{T}}$ which are also in $y$, i.e.
\[ S_{E,y} \ = \ \{ \ell \in L_{v,E,\mathcal{T}}  \mid \ell \in y \} \ = \ y \cap L_{v,E,\mathcal{T}} 
\]

\begin{lemma}  \label{lem:abcde}
Suppose that $E,E'$ are both weak d-AXpc's such that $E'$ strictly subsumes $E$ in $\mathcal{T}$
(cov$_{\mathcal{T}}$($E$) $\subset$ cov$_{\mathcal{T}}$($E'$)) and $y \in \text{cov}_{\mathcal{T}}(E') \setminus \text{cov}_{\mathcal{T}}(E)$. Then
\begin{itemize}
\item[(a)] $E' \subseteq L_{v,E,\mathcal{T}}$
\item[(b)] $E' \subseteq S_{E,y}$
\item[(c)] $S_{E,y}$ is a weak d-AXp
\item[(d)] cov$_{\mathcal{T}}$($E$) $\subseteq$ cov$_{\mathcal{T}}$($S_{E,y})$
\item[(e)] $y \in \text{cov}_{\mathcal{T}}(S_{E,y}) \setminus \text{cov}_{\mathcal{T}}(E)$
\end{itemize}
\end{lemma}

\begin{proof}
(a) From the definition of $L_{v,E,\mathcal{T}}$ and cov$_{\mathcal{T}}$($E$) $\subset$ cov$_{\mathcal{T}}$($E'$), 
we have $L_{v,E',\mathcal{T}} \subseteq L_{v,E,\mathcal{T}}$.
Since $E' \subseteq L_{v,E',\mathcal{T}}$ (by definition of $L_{v,E',\mathcal{T}}$), we have $E' \subseteq L_{v,E,\mathcal{T}}$.
(b) Let $\ell \in L_{v,E,\mathcal{T}} \setminus S_{E,y}$. Then, by definition of $S_{E,y}$, $\ell \notin y$. 
But, since $y \in \text{cov}_{\mathcal{T}}(E')$,
we can deduce that $\ell \notin E'$. It follows from (a) that $E' \subseteq S_{E,y}$.
(c) This follows from (b) and the fact that $E'$ is a weak d-AXp.
(d) Suppose $z \in \text{cov}_{\mathcal{T}}(E)$ and consider 
any $\ell \in S_{E,y}$. Then, by definition of $S_{E,y}$,
$\ell \in L_{v,E,\mathcal{T}}$ and so $\ell \in z$ (by equation~(\ref{eq:LvET})). 
It follows that $z \in \text{cov}_{\mathcal{T}}(S_{E,y})$ $=$
$\{w \in \mathcal{T} \mid S_{E,y} \subseteq w\}$.
(e) By definition of $S_{E,y}$, for all $\ell \in S_{E,y}$, $\ell \in y$ and hence $y \in \text{cov}_{\mathcal{T}}(S_{E,y})$.
By the assumption $y \in \text{cov}_{\mathcal{T}}(E') \setminus \text{cov}_{\mathcal{T}}(E)$, 
we already know that $y \notin \text{cov}_{\mathcal{T}}(E)$. 
\end{proof}
\begin{mytheorem} 
Let $E \in \mbb{E}$. 
\begin{enumerate}
	\item Testing whether $E$ is a d-CPI-Xp can be achieved in $O(m^2n)$ time.
	\item Finding a d-CPI-Xp can be achieved in $O(m^2n^2)$ time.
	\item Finding a minimal d-CPI-Xp can be achieved in $O(m^2n^2)$ time. 
	\item Finding a preferred d-CPI-Xp can be achieved in $O(m^2n^2)$ time. 
\end{enumerate} 
\end{mytheorem}

\begin{proof}
Let us show the \emph{first} result. 

As we have already observed, 
we can test whether $E$ is a weak d-AXp is $O(mn)$ time.
Lemma~\ref{lem:abcde} tells us that if $E$ is strictly subsumed by another weak d-AXp $E'$,
then, for some $y \in \mathcal{T} \setminus \text{cov}_{\mathcal{T}}(E)$, 
$E$ is strictly subsumed by $S_{E,y}$. 
The set $L_{v,E,\mathcal{T}}$ 
can be calculated in $O(mn)$ time, and then we can calculate 
$S_{E,y}$ for every $y \in \mathcal{T} \setminus \text{cov}_{\mathcal{T}}(E)$ in $O(mn)$ total time.
For each such $y$, checking whether $S_{E,y}$ is a weak d-AXp can be achieved in $O(mn)$ time,
which gives a total $O(m^2n)$ time complexity.

\vspace{0.2cm}
Let us show the \emph{second} result.

The algorithm outlined in the proof of the first result above
either confirms that a
weak d-AXp $E$ is a d-CPI-Xp or finds a weak d-AXp $E'$ such that 
cov$_{\mathcal{T}}$($E$) $\subset$ cov$_{\mathcal{T}}$($E'$).
So we can initialise $E$ to $v$ (which is necessarily a weak d-AXp) and then
iterate until $E$ is a d-CPI-Xp, replacing $E$ by $E'$ 
(a weak d-AXp which strictly subsumes $E$)
when $E$ is not a d-CPI-Xp. 

Now, suppose that cov$_{\mathcal{T}}$($E$) $\subset$ 
cov$_{\mathcal{T}}$($E'$). Then $L_{v,E',\mathcal{T}} \subseteq L_{v,E,\mathcal{T}}$
by definition of $L_{v,E,\mathcal{T}}$. But there also exists some $y \in \text{cov}_{\mathcal{T}}(E') \setminus \text{cov}_{\mathcal{T}}(E)$.
It follows that there is some literal $\ell \in E \setminus E'$ such that $\ell \notin y$, hence
$\ell \notin L_{v,E',\mathcal{T}}$ and $\ell \in L_{v,E,\mathcal{T}}$. Therefore $L_{v,E',\mathcal{T}} \subset L_{v,E,\mathcal{T}}$.
Thus the size of $L_{v,E,\mathcal{T}} \subseteq v$ 
decreases strictly at each iteration. This completes
the proof, since the number of iterations is therefore necessarily $O(n)$.

\vspace{0.2cm}
Let us show the \emph{third} result.

Starting from a d-CPI-Xp $E$, we can find a minimal d-CPI-Xp by the standard deletion
algorithm which, for each element of $E$ in turn, deletes it if $E$ remains a weak d-AXp after
its deletion. Note that the resulting set $E'$ must have exactly the same coverage as $E$;
it cannot be smaller since $E' \subset E$ and it cannot be larger since then $E'$ would
strictly subsume $E$ (contradicting the definition of a d-CPI-Xp). 
This implies that $E'$ remains a d-CPI-Xp.
This deletion algorithm may require another $O(n)$ calls to the $O(mn)$ algorithm for
deciding whether a set is a weak d-AXp. Hence the asymptotic time complexity is $O(m^2n^2)$,
the complexity of finding the initial d-CPI-Xp.

\vspace{0.2cm}
To show the \emph{fourth} result it suffices to note that
finding one preferred d-CPI-Xp consists of finding one minimal 
minimal d-CPI-Xp. 
\end{proof}

We now justify the complexities given in Table~\ref{tab:abductive}.

The polynomial-time complexities for testing
and finding dataset-based explanations follow
from Theorem~\ref{thm:testdxAXp} or Theorem~\ref{prop:testdCPIXp}. Observe that testing
subset-minimality, in the case of d-AXp or d-mCPI-Xp,
requires up to $n$ extra tests that removing a literal
does not leave a d-wAXp or d-CPI-Xp, respectively.
Hence complexity is again polynomial. Since a d-pCPI-Xp
is any d-mCPI-Xp, the complexity of finding a d-pCPI-Xp
is identical to that of finding a d-mCPI-Xp.

The complexities of testing whether a set $E$
is a weak AXp and of finding an AXp are given by
Property 1. The complexity of finding a weak AXp
is trivially polynomial since $x$ is a weak AXp
of the decision $\kappa(x)$.
To show that a set $E$ is an AXp we need to show that
it is a weak AXp and that each of the (up to $n$) sets obtained
by deleting an element from $E$ are not weak AXp's.
Hence at most $n+1$ calls to an NP oracle suffice.

A CPI-Xp is a weak AXpc that is not strictly subsumed.
Testing whether a set $E$ is a weak AXpc is clearly in co-NP
since a counter-example can be verified in polynomial time.
Theorem~\ref{prop:pi2} tells us that the complexity of testing
whether a set $E$ is a CPI-Xp is dominated by the complexity
of testing that it is not strictly subsumed by a weak AXpc,
which is $\Pi_2^P$-complete. Testing whether a set $E$
is an mCPI-Xp only requires up to $n$ extra tests that
removing a literal does not leave a weak AXpc and hence does
not change the $\Pi_2^P$-completeness. 
Theorem~\ref{prop:finding} and
Theorem~\ref{prop:minimalCPI} tell us that a polynomial
number of calls to a $\Sigma_2^{\text{P}}$ oracle 
(which is the same thing as a $\Pi_2^P$ oracle) are
sufficient to find a CPI-XP, mCPI-XP or pCPI-Xp.

\begin{mytheorem} 
The properties are compatible. 
\end{mytheorem}

\begin{proof}
Table~\ref{tab} shows that the function $\bF_{pCPI-Xp}$ satisfies all the properties. 
\end{proof}

\begin{mytheorem}
	The properties of Table~\ref{tab} hold.
\end{mytheorem}

\begin{proof}
Let $\kappa$ be a classifier.

\paragraph{Success.} Let $x \in \mbb{F}$. 

It has been shown in \cite{Amgoud21}, that $\bF_{AXp}(x) \neq \emptyset$, thus $\bF_{AXp}$ satisfies Success. From the inclusion $\bF_{AXp}(x) \subseteq \bF_{wAXp}(x)$ (in Proposition~\ref{links}), it follows that $\bF_{wAXp}$ also satisfies Success.
From Proposition~\ref{axp-vs-axpc} and the fact that $\bF_{AXp}(x) \neq \emptyset$, it follows 
that $\bF_{AXpc}(x) \neq \emptyset$ meaning that $\bF_{AXpc}$ satisfies Success. 
From the inclusion $\bF_{AXpc}(x) \subseteq \bF_{wAXpc}(x)$ (by definition), it follows that $\bF_{wAXpc}$ also satisfies Success.
It has also been shown in \cite{Amgoud21b} that $\bF_{d-AXp}(x) \neq \emptyset$. 
Since $\bF_{d-AXp}(x) \subseteq \bF_{d-wAXp}(x)$, then $\bF_{d-wAXp}(x) \neq \emptyset$.

Let us now show that $\bF_{CPI-Xp}$ satisfies success. 
Recall that $E \in \mbb{E}$ is a CPI-Xp of $x$ if the following conditions hold: 
\begin{description}
\item [a)] $E(x)$, 
\item [b)] $\forall y \in \mbb{F}[\C]. (E(y) \rightarrow (\kappa(y)=\kappa(x)))$, 
\item [c)] $\nexists E' \in \mbb{E}$ such that $E'$ 
           satisfies the conditions a) and b) and strictly subsumes $E$.  
\end{description}
Let $T = \{E \in \mbb{E} \ | \ E \mbox{ satisfies a), b)}\}$.
Assume $\bF_{CPI-Xp}(x) = \emptyset$. Elements of $T$ are 
wAXpc's of $\kappa(x)$ (from Proposition~\ref{links}), 
so from Success of $\bF_{wAXpc}$, we have 
$T \neq \emptyset$. 
Since $x$ is finite (due to the finiteness of $\F$), then 
$T$ is also finite. Let $T = \{E_1, \ldots, E_k\}$.
Since $\bF_{CPI-Xp}(x) = \emptyset$, then $\forall i \in \{1, \ldots, k\}$, 
$E_i$ violates c), so $\exists j \in \{1, \ldots, k\}$ such that 
$E_j$ strictly subsumes $E_i$ in $\mbb{F}$, and from Proposition~\ref{cov}, 
$\cov(E_i) \subset \cov(E_j)$.  
Thus, there exists a permutation $\rho$ on $\{1, \ldots, k\}$ such that for 
$T = \{E_{\rho_1}, \ldots, E_{\rho_k}\}$, we have 
$\cov(E_{\rho_1}) \subset \ldots \subset \cov(E_{\rho_k})$ 
and $E_{\rho_k}$ strictly subsumes an element in $T$, say 
$E_{\rho_j}$ with $j < k$. It follows that 
$\cov(E_{\rho_j}) = \cov(E_{\rho_k})$, which contradicts the assumption. 

By definition, mCPI-Xp's of $\kappa(x)$ are subset-minimal CPI-Xp's of $\kappa(x)$. 
Since $\bF_{CPI-Xp}(x) \neq \emptyset$, it contains at least one 
element which is minimal. 
$\bF_{pCPI-Xp}(x)$ contains a subset of $\bF_{mCPI-Xp}(x)$ where only one element per representative 
class is chosen. Since  $\bF_{mCPI-Xp}(x) \neq \emptyset$, $\bF_{pCPI-Xp}(x) \neq \emptyset$. 
The same reasoning holds for the three sample-based versions. 

\vspace{0.2cm}

\paragraph{Non-Triviality.} Let $x \in \mbb{F}[\C]$ s.t. $\kappa(x) = c$. 
Assume that $\emptyset$ is a weak AXp of $\kappa(x) = c$.  
So, $\{x \in \mbb{F} \ | \ \kappa(x) = c\} = \mbb{F}$, which  
contradicts the assumption that $\kappa$ is a surjective function in the non-constrained setting and 
$|\mathtt{Cl}| \geq 2$. Hence, $\emptyset \notin \bF_{wAXp}(x)$. 
From the inclusion $\bF_{AXp}(x) \subseteq \bF_{wAXp}(x)$ (Prop.~\ref{links}), 
it follows that $\emptyset \notin \bF_{AXp}(x)$. 
Assume that $\emptyset$ is a weak AXpc of $\kappa(x) = c$.  
So, $\{x \in \mbb{F} \mid \C(x) \wedge \kappa(x) = c\} = \mbb{F}[\C]$, which  
contradicts the assumption that $\kappa$ is a surjective function in the 
constrained setting and the set $|\mathtt{Cl}| \geq 2$. 
Hence, $\emptyset \notin \bF_{wAXpc}(x)$. 
From the inclusion $\bF_{AXpc}(x) \subseteq \bF_{wAXpc}(x)$ (Prop.~\ref{links}), 
it follows that $\emptyset \notin \bF_{AXpc}(x)$. 
Assume now that $\emptyset$ is a CPI-Xp of $\kappa(x) = c$. 
Then $\{x \in \mbb{F} \ | \ \kappa(x) = c\} = \mbb{F}[\C]$ 
and $\mbb{F}[\C] \neq \emptyset$ (assumption \textbf{(C1)} on $\C$).
This contradicts the assumption that $\kappa$ is a surjective function in the constrained 
setting and $|\mathtt{Cl}| \geq 2$. 
From the inclusions  $\bF_{pCPI-Xp}(x) \subseteq \bF_{mCPI-Xp}(x) \subseteq \bF_{CPI-Xp}(x)$ (Prop.~\ref{links}), it follows that 
$\emptyset \notin \bF_{pCPI-Xp}(x)$ and $\emptyset \notin \bF_{mCPI-Xp}(x)$.

Non-Triviality of the sample-based functions follows straightforwardly from the assumption 
$\forall c \in \mathtt{Cl}, \exists x \in \fml{T} \mbox{ such that } \kappa(x) = c$ on any 
sample $\fml{T}$. 

\vspace{0.2cm}

\paragraph{Irreducibility.}
			Consider the following feature space $\mbb{F}[\C]$ which is
			simply the set of all boolean vectors $(A,B,C)$ subject to the
			constraint $A=B$.
		
			\begin{nolinenumbers}
            \begin{multicols}{2}
	
				\begin{tabular}{|c|ccc|c|}\hline
					&	$A$ & $B$  & $C$  &  $\kappa(x_i)$ \\\hline
					$x_1$      &	0   & 0    & 0    &   0    \\
					$x_2$      &	0   & 0    & 1    &   0    \\
				\rowcolor{maroon!20}	$x_3$      &    1   & 1    & 0    &   1    \\
					$x_4$      &	1   & 1    & 1    &   1    \\ \hline
				\end{tabular}
		
			\begin{itemize}
			\item $E_1 = \{(A,1)\}$ 
			\item $E_2 = \{(B,1)\}$ 
			\item $E_3 = \{(A,1),(B,1)\}$ 
			\end{itemize}
\end{multicols}
\end{nolinenumbers}
Note that $E_1, E_2, E_3$ are CPI-Xp's of $\kappa(x_3)$. 
Clearly $E_3$ is not irreducible, so $\bF_{CPI-Xp}$ does not satisfy the Irreducibility property. From the inclusion $\bF_{CPI-Xp}(x) \subseteq \bF_{wAXpc}$ (Prop.~\ref{links}), it follows that $\bF_{wAXpc}$ violates 
Irreducibility. 
Assume now a sample $\fml{T}$  made of exactly the above instances. Note that 
$E_1, E_2, E_3$ are d-wAXp's and d-CPI-Xp's of $\kappa(x_3)$. This shows 
that $\bF_{d-wAXp}$ and $\bF_{d-CPI-Xp}$ violate Irreducibility. 

Example~1 shows that $\bF_{wAXp}$ violates Irreducibility. Indeed, 
the decision $\kappa(x_4)$ has three weak AXps: 
$E_1 = \{(f_1,1)\}$, 
$E_2 = \{(f_2,1)\}$ and $E_3 = \{(f_1,1),(f_2,1)\}$.

We now show that an explanation function which returns subset-minimal explanations satisfies Irreducibility. 

Let $E$ be an AXpc of $\kappa(x) = c$ where $x \in \mbb{F}[\C]$. 
Let $l \in E$. Obviously, $(E\setminus\{l\})(x)$ holds. Assume that 
$\forall y \in \mbb{F}[\C]$, $(E\setminus\{l\})(x) \rightarrow \kappa(y) = c$. 
This means that $E\setminus\{l\}$ is an AXpc, which contradicts the minimality of $E$. So, $\exists x' \in \mbb{F}[\C]$ s.t. $(E\setminus\{l\})(x')$ and $\kappa(x') \neq c$. 

We can show in a similar way the property for the other types of explanations that satisfy the property of minimality for set-inclusion.

\vspace{0.2cm}
\paragraph{Coherence.} 
Let $x, x' \in \mbb{F}[\C]$ s.t. $\kappa(x) \neq \kappa(x')$. 
Let $E$ be a weak AXpc of $\kappa(x)$ and $E'$ a weak AXpc of 
$\kappa(x')$. Assume that $\exists x'' \in \mbb{F}[\C]$ s.t. $(E \cup E')(x'')$. Since $E \subseteq E \cup E'$, we have $E(x'')$ and so $\kappa(x'') = \kappa(x)$. Similarly, $E'(x'')$ holds and $\kappa(x'') = \kappa(x')$, which is impossible. Therefore $\bF_{wAXpc}$ satisfies coherence. 
From Proposition~\ref{links}, it follows that 
$\bF_{AXpc}$, $\bF_{CPI-Xp}$, $\bF_{mCPI-Xp}$ and $\bF_{pCPI-Xp}$ satisfy coherence. 
It has been shown in \cite{Amgoud21b} that $\bF_{wAXp}$ and $\bF_{AXp}$ satisfy coherence. 

\vspace{0.2cm}		      
We will now show that explainers based on datasets violate Coherence.  Consider a classifier $\kappa$ that predicts whether a second-hand car is expensive or not (thus, $\mathtt{Cl} = \{0,1\}$) based on three features: the power (P) of its engine, its colour (C), the state of its equipment (S), with 
 $\d(P) = \d(S) = \{\mbox{High}, \mbox{Medium}, \mbox{Low}\}$ and $\d(C)$ is the set of colours. The table below gives the predictions of $\kappa$ for a sample  $\fml{T}$ of cars. 
	 
\begin{center}
\begin{tabular}{|c|ccc|c|}
 			\hline
 			&	$P$             &  $C$  &  $S$  & $\kappa(x_i)$ \\	\hline 
 		\rowcolor{maroon!20}	$x_1$	&	$\mbox{High}$   & White &  $\mbox{High}$      & 1   \\
 			$x_2$	&	$\mbox{High}$   & White &  $\mbox{Medium}$    & 1   \\
 			$x_3$	&	$\mbox{Medium}$ & Red   &  $\mbox{Low}$       & 0   \\	
 			$x_4$	&	$\mbox{High}$   & Blue  &  $\mbox{Low}$       & 1   \\	\hline
 		\end{tabular}
\end{center}

Consider the following partial assignments: 
$E_1 = \{(P,High)\}$ and $E_2 = \{(C, Red)\}$. 
Note that $E_1$ is a d-CPI-Xp of $\kappa(x_1) = 1$ and 
          $E_2$ is a d-CPI-Xp of $\kappa(x_3) = 0$. 
Assume that there exists an instance $x_5 = \{(P,High), (C, Red), (S,v)\}$
in $\mbb{F}[\C] \setminus \mathcal{T}$, 
with $v \in \d(S)$ and which satisfies all the constraints in $\C$. 
Then, $\bF_{d-CPI-Xp}$ violates coherence. The same counter-example applies for showing that $\bF_{d-mCPI-Xp}$, $\bF_{d-pCPI-Xp}$, 
$\bF_{d-wAXp}$ and $\bF_{d-AXp}$ violate coherence.


\vspace{0.2cm}
		      
\paragraph{Consistency.} We show that all explainers satisfy consistency. 
%
%
%
Let $\bF_{z}$ denote any of the reviewed functions. 
Since by definition, $x \in \mbb{F}[\C]$ we know that $x$ satisfies every 
$c \in \C$. From condition \textbf{(C2)}, this property is inherited by 
every subset of $x$. Indeed, for any $E \subseteq x$, $E$ satisfies 
every $c \in \C$. Since for any $E \in \bF_{z}(x)$ $E(x)$, then $E$ 
satisfies all the constraints in $\C$ and so the function $\bF_{z}$ satisfies consistency. 

\paragraph{Independence.} 
Let us show that $\bF_{wAXp}$, $\bF_{AXp}$, $\bF_{AXpc}$, $\bF_{d-wAXp}$ and $\bF_{d-AXp}$ 
violate independence. It is sufficient to consider the following example. 

Assume that $\F = \{f_1, f_2\}$, with $\d(f_1) = \d(f_2) = \{0,1\}$, and $\mathtt{Cl} = \{0,1\}$.  
Assume the functional dependency $\{(f_1,1)\} \to \{(f_2,1)\}$.
Consider the classifier $\kappa$ such that for any $x \in \mbb{F}$, $\kappa(x) = (f_1,1) \vee (f_2,1)$. Its predictions are summarized in the table below. 

\begin{nolinenumbers}
\begin{multicols}{2}
	\begin{center}
		\begin{tabular}{|c|cc|c|}
		\hline
			&$f_1$ & $f_2$ & $\kappa(.)$ \\
			\hline 
            $x_1$ & 0 & 0 &  0 \\
			$x_2$ & 0 & 1 &  1 \\
			$x_3$ & 1 & 0 &  1 \\
\rowcolor{maroon!20}$x_4$ & 1 & 1 &  1 \\
		\hline
		\end{tabular}
	\end{center}

\begin{itemize}
	\item $E_1 = \{(f_1,1)\}$ 
	\item $E_2 = \{(f_2,1)\}$
	\item $E_3 = \{(f_1,1),(f_2,1)\}$ 
\end{itemize}	
\end{multicols}
\end{nolinenumbers}

\noindent The decision $\kappa(x_4)$ has three wAXp's/wAXpc's ($E_1, E_2, E_3$) and two AXp's/AXpc's ($E_1, E_2$). Note that $E_1 \to E_2 \in \C^*$ and $E_2 \to E_1 \notin \C^*$. 
Which shows that the four functions violate independence. 

Consider now the sample $\mathcal{T} = \{x_1, x_2, x_4\}$.  
The decision $\kappa(x_4)$ has $E_1, E_2, E_3$ as d-wAXp's/d-wAXpc's and $E_1, E_2$ as d-AXp's/d-AXpc's. This shows that the four sample-based functions violate independence.  

Consider now the sample $\mathcal{T'} = \{x_1, x_4\}$. 
The decision $\kappa(x_4)$ has $E_1, E_2, E_3$ as CPI-Xp's and $E_1, E_2$ as mCPI-Xp's. This shows that $\bF_{d-CPI-Xp}$ and $\bF_{d-mCPI-Xp}$ violate 
independence. 

\vspace{0.2cm}

Let us show that the functions $\bF_{CPI-Xp}$, $\bF_{mCPI-Xp}$ and $\bF_{pCPI-Xp}$ satisfy independence. 
Assume that $\bF_{CPI-Xp}$ violates independence, hence 
$\exists x \in \mbb{F}[\C]$ and 
$\exists E, E' \in \bF_{CPI-Xp}(x)$ such that $E \to E' \in \C^*$ and 
$E' \to E \notin \C^*$. 
From Proposition~\ref{cov-cont}, $E'$ subsumes $E$ in $\mbb{F}[\C]$, 
so from Proposition~\ref{cov}, 
$\cov(E) \subseteq \cov(E')$. 
By definition of a CPI-Xp, $E$ does not strictly subsume $E'$ in $\mbb{F}[\C]$ and 
$E'$ does not strictly subsume $E$ in $\mbb{F}[\C]$. Hence, $\cov(E) = \cov(E')$ (1). 
Let $Z = \{z \in \mbb{F} \ | \ E'(z) \wedge \neg E(z)\}$.  
By definition of $\mbb{F}$, $Z \neq \emptyset$.
Let $z \in Z$. Assume $\C(z)$, hence $\kappa(z) = c$ (since $E'$ is a CPI-Xp). 
So, $\cov(E) \subset \cov(E')$, which contradicts (1). 
Then, $\forall z \in \mbb{F}$, $E'(z) \wedge \neg E(z) \to \neg \C(z)$. 
So, $E'(z) \wedge \C(z) \to E(z)$. It follows that, 
$\forall z \in \mbb{F}$ such that $\C(z)$, $E'(z) \to E(z)$, then 
$E' \to E \in \C^*$, which contradicts the assumption.

From the inclusions $\bF_{pCPI-Xp}(x) \subseteq \bF_{mCPI-Xp}(x) \subseteq \bF_{CPI-Xp}(x)$, it follows that $\bF_{pCPI-Xp}$ and $\bF_{mCPI-Xp}$ satisfy  independence. 
\vspace{0.2cm}

\paragraph{Non-Equivalence.} Functions that generate preferred CPI-Xp and sample-based preferred CPI-Xp satisfy Non-Equivalence by definition (since only one representative is considered per equivalence class).

\vspace{0.2cm}

Let us now show that the two explanation functions that produce weak AXp's and AXp's satisfy non-equivalence. Due to the inclusion $\bF_{AXp}(x) \subseteq \bF_{wAXp}(x)$ for every $x \in \mbb{F}$, it is sufficient to show the property for $\bF_{wAXp}$. 
Assume $E, E' \in \bF_{wAXp}(x)$ with $E \neq E'$. 
From the last property in 
Proposition~\ref{cov}, $\cov_{\mbb{F}}(E) \neq \cov_{\mbb{F}}(E')$ and 
thus $E \not\approx E'$. 
\vspace{0.2cm}

Let us show that the remaining functions violate Non-Equivalence.
Consider again the feature space with the features ($A$, $B$, $C$) introduced above when studying Irreducibility. Recall that there is a constraint $A=B$ and that both
$E_1 = \{(A,1)\}$ and $E_2 = \{(B,1)\}$ are CPI-Xp's (resp. AXpc's) 
of $\kappa(x_3)$.  Note that $\cov(E_1) = \cov(E_2) = \{x_3, x_4\}$. 
Hence, these two explanations are equivalent. 
The same example applies for the all other types of explanations except 
pCPI-Xp and d-pCPI-Xp. 
\end{proof}

\end{appendix}
\bibliography{refsExplanations}

\begin{thebibliography}{10}

\bibitem{Amgoud21}
Leila Amgoud, `Explaining black-box classification models with arguments', in
  {\em {ICTAI} 2021}, pp. 791--795, (2021).

\bibitem{Amgoud21b}
Leila Amgoud, `Non-monotonic explanation functions', in {\em Proceedings of the
  16th European Conference on Symbolic and Quantitative Approaches to Reasoning
  with Uncertainty, {ECSQARU}}, volume 12897 of {\em Lecture Notes in Computer
  Science}, pp. 19--31. Springer, (2021).

\bibitem{Amgoud23a}
Leila Amgoud, `Explaining black-box classifiers: Properties and functions',
  {\em International Journal of Approximate Reasoning}, {\bf 155},  40--65,
  (2023).

\bibitem{Amgoud22}
Leila Amgoud and Jonathan Ben-Naim, `Axiomatic foundations of explainability',
  in {\em {IJCAI}}, pp. 636--642, (2022).

\bibitem{Amgoud23b}
Leila Amgoud, Philippe Muller, and Henri Trenquier, `Leveraging argumentation
  for generating robust sample-based explanations', in {\em {IJCAI}}, p. In
  press, (2023).

\bibitem{marquis22a}
Gilles Audemard, Steve Bellart, Louenas Bounia, Frederic Koriche, Jean-Marie
  Lagniez, and Pierre Marquis, `On preferred abductive explanations for
  decision trees and random forests', in {\em Proceedings of the Twenty-Seventh
  International Joint Conference on Artificial Intelligence, {IJCAI}}, pp.
  643--650, (2022).

\bibitem{AudemardBBKLM22}
Gilles Audemard, Steve Bellart, Louenas Bounia, Fr{\'{e}}d{\'{e}}ric Koriche,
  Jean{-}Marie Lagniez, and Pierre Marquis, `Trading complexity for sparsity in
  random forest explanations', in {\em Thirty-Sixth {AAAI} Conference on
  Artificial Intelligence, {AAAI}}, pp. 5461--5469. {AAAI} Press, (2022).

\bibitem{AudemardLMS23}
Gilles Audemard, Jean{-}Marie Lagniez, Pierre Marquis, and Nicolas Szczepanski,
  `Computing abductive explanations for boosted trees', in {\em International
  Conference on Artificial Intelligence and Statistics}, volume 206 of {\em
  Proceedings of Machine Learning Research}, pp. 4699--4711, (2023).

\bibitem{BaileyS05}
James Bailey and Peter~J. Stuckey, `Discovery of minimal unsatisfiable subsets
  of constraints using hitting set dualization', in {\em Practical Aspects of
  Declarative Languages, 7th International Symposium, {PADL}}, eds., Manuel~V.
  Hermenegildo and Daniel Cabeza, volume 3350, pp. 174--186. Springer, (2005).

\bibitem{ChenT95}
Zhi{-}Zhong Chen and Seinosuke Toda, `The complexity of selecting maximal
  solutions', {\em Inf. Comput.}, {\bf 119}(2),  231--239, (1995).

\bibitem{CP21}
Martin~C. Cooper and Jo{\~{a}}o Marques{-}Silva, `On the tractability of
  explaining decisions of classifiers', in {\em {CP} 2021}, ed., Laurent~D.
  Michel, volume 210 of {\em LIPIcs}, pp. 21:1--21:18. Schloss Dagstuhl -
  Leibniz-Zentrum f{\"{u}}r Informatik, (2021).

\bibitem{COOPER2023}
Martin~C. Cooper and João Marques-Silva, `Tractability of explaining
  classifier decisions', {\em Artificial Intelligence}, {\bf 316},  103841,
  (2023).

\bibitem{darwiche20}
Adnan Darwiche and Auguste Hirth, `On the reasons behind decisions', in {\em
  24th European Conference on Artificial Intelligence {ECAI}}, volume 325, pp.
  712--720. {IOS} Press, (2020).

\bibitem{marquis22b}
Alexis de~Colnet and Pierre Marquis, `On the complexity of enumerating prime
  implicants from decision-dnnf circuits', in {\em Proceedings of the
  Thirty-First International Joint Conference on Artificial Intelligence,
  {IJCAI}}, ed., Luc~De Raedt, pp. 2583--2590, (2022).

\bibitem{BandaSW03}
Maria J.~Garc{\'{\i}}a de~la Banda, Peter~J. Stuckey, and Jeremy Wazny,
  `Finding all minimal unsatisfiable subsets', in {\em Proceedings of the 5th
  International {ACM} {SIGPLAN} Conference on Principles and Practice of
  Declarative Programming}, pp. 32--43. {ACM}, (2003).

\bibitem{GorjiAAAI22}
Niku Gorji and Sasha Rubin, `Sufficient reasons for classifier decisions in the
  presence of domain constraints', in {\em {AAAI}}, pp. 5660--5667, (2022).

\bibitem{IgnatievNM19}
Alexey Ignatiev, Nina Narodytska, and Jo{\~{a}}o Marques{-}Silva,
  `Abduction-based explanations for machine learning models', in {\em {AAAI}
  2019}, pp. 1511--1519. {AAAI} Press, (2019).

\bibitem{joao19b}
Alexey Ignatiev, Nina Narodytska, and Joao Marques{-}Silva, `On relating
  explanations and adversarial examples', in {\em Thirty-third Conference on
  Neural Information Processing Systems, {NeurIPS}}, pp. 15857--15867, (2019).

\bibitem{IgnatievS21}
Alexey Ignatiev and Jo{\~{a}}o P.~Marques Silva, `Sat-based rigorous
  explanations for decision lists', in {\em Theory and Applications of
  Satisfiability Testing - {SAT}}, eds., Chu{-}Min Li and Felip Many{\`{a}},
  volume 12831, pp. 251--269. Springer, (2021).

\bibitem{Junker04}
Ulrich Junker, `{QUICKXPLAIN:} preferred explanations and relaxations for
  over-constrained problems', in {\em Proceedings of the Nineteenth National
  Conference on Artificial Intelligence {AAAI}}, eds., Deborah~L. McGuinness
  and George Ferguson, pp. 167--172, (2004).

\bibitem{LiffitonS05}
Mark~H. Liffiton and Karem~A. Sakallah, `On finding all minimally unsatisfiable
  subformulas', in {\em Theory and Applications of Satisfiability Testing, 8th
  International Conference, {SAT}}, eds., Fahiem Bacchus and Toby Walsh, volume
  3569 of {\em Lecture Notes in Computer Science}, pp. 173--186. Springer,
  (2005).

\bibitem{LiuL23}
Xinghan Liu and Emiliano Lorini, `A unified logical framework for explanations
  in classifier systems', {\em Journal of Logic and Computation}, {\bf 33}(2),
  485--515, (2023).

\bibitem{Miller56}
George~A. Miller, `Some limits on our capacity for processing information',
  {\em Psychological review}, {\bf 63}(2),  81--97, (1956).

\bibitem{Miller2019}
Tim Miller, `Explanation in artificial intelligence: Insights from the social
  sciences', {\em Artificial Intelligence}, {\bf 267},  1--38, (2019).

\bibitem{molnar2020}
C.~Molnar, {\em Interpretable Machine Learning}, Lulu.com, 2020.

\bibitem{Ribeiro16}
Marco~T{\'{u}}lio Ribeiro, Sameer Singh, and Carlos Guestrin, `Why should {I}
  trust you?: Explaining the predictions of any classifier', in {\em
  Proceedings of the 22nd ACM SIGKDD International Conference on Knowledge
  Discovery and Data Mining}, p. 1135–1144, (2016).

\bibitem{Ribeiro0G18}
Marco~T{\'{u}}lio Ribeiro, Sameer Singh, and Carlos Guestrin, `Anchors:
  High-precision model-agnostic explanations', in {\em Proceedings of the
  Thirty-Second {AAAI} Conference on Artificial Intelligence, (AAAI-18)}, pp.
  1527--1535, (2018).

\bibitem{darwiche18}
Andy Shih, Arthur Choi, and Adnan Darwiche, `A symbolic approach to explaining
  {B}ayesian network classifiers', in {\em Proceedings of the Twenty-Seventh
  International Joint Conference on Artificial Intelligence, {IJCAI}}, pp.
  5103--5111, (2018).

\bibitem{Thalheim}
Bernhard Thalheim, {\em Dependencies in Relational Databases}, Vieweg+Teubner
  Verlag Wiesbaden, 1991.

\end{thebibliography}

\end{document}